\documentclass{article}
\usepackage[margin=1in]{geometry}

\usepackage{mymath,natbib}
\title{
\LARGE
Large Stepsize Gradient Descent for Non-Homogeneous\\ Two-Layer Networks: Margin Improvement and Fast Optimization
}

\author{
Yuhang Cai\thanks{UC Berkeley. Email: \texttt{willcai@berkeley.edu}}
  \and
Jingfeng Wu\thanks{UC Berkeley. Email: \texttt{uuujf@berkeley.edu}}
  \and
  Song Mei\thanks{UC Berkeley. Email: \texttt{songmei@berkeley.edu}}
  \and
  Michael Lindsey\thanks{UC Berkeley and Lawrence Berkeley National Laboratory. Email: \texttt{lindsey@berkeley.edu}}
  \and
  Peter L. Bartlett\thanks{UC Berkeley and Google DeepMind. Email: \texttt{peter@berkeley.edu}}
}

\date{\today}

\begin{document}

\maketitle

\begin{abstract}
The typical training of neural networks using large stepsize gradient descent (GD) under the logistic loss often involves two distinct phases, where the empirical risk oscillates in the first phase but decreases monotonically in the second phase. We investigate this phenomenon in two-layer networks 
that satisfy a near-homogeneity condition.
We show that the second phase begins once the empirical risk falls below a certain threshold, dependent on the stepsize. Additionally, we show that the normalized margin grows nearly monotonically in the second phase, 
demonstrating
an implicit bias of GD in training non-homogeneous predictors. If the dataset is linearly separable and the derivative of the activation function is bounded away from zero, we show that the average empirical risk decreases, implying that the first phase must stop in finite steps. Finally, we demonstrate that by choosing a suitably large stepsize, GD that undergoes this phase transition is more efficient than GD that monotonically decreases the risk.
Our analysis applies to networks of any width, beyond the well-known neural tangent kernel and mean-field regimes.

\end{abstract}

\section{Introduction}

Neural networks are mostly optimized by {\it gradient descent} (GD) or its variants. 
Understanding the behavior of GD is one of the key challenges in deep learning theory. 
However, there is a nonnegligible discrepancy between the GD setups in theory and in practice.
In theory, GD is mostly analyzed with relatively small stepsizes such that its dynamics are close to the continuous \emph{gradient flow} dynamics, although a few exceptions will be discussed later. 
However, in practice, GD is often used with a relatively large stepsize, with behaviors significantly deviating from that of small stepsize GD or gradient flow. 
Specifically, notice that small stepsize GD (hence also gradient flow) induces monotonically decreasing empirical risk, but in practice, good optimization and generalization performance is usually achieved when the stepsize is large and the empirical risk {oscillates} \citep[see][for example]{wu2018sgd,cohen2020gradient}.
Therefore, it is unclear which of the theoretical insights drawn from analyzing small stepsize GD apply to large stepsize GD used practically. 

The behavior of small stepsize GD is relatively well understood. For instance, classical optimization theory suggests that GD minimizes convex and $L$-smooth functions if the stepsize $\tilde\eta$ is well below $2/L$, with a convergence rate of $\Ocal(1/(\tilde\eta t))$, where $t$ is the number of steps \citep{nesterov2018lectures}. More recently, \citet{soudry2018implicit,ji2018risk} show an \emph{implicit bias} of small stepsize GD in logistic regression with separable data, where the direction of the GD iterates converges to the max-margin direction. 
Subsequent works extend their implicit bias theory from linear model to \emph{homogenous} networks \citep{lyu2020gradient,chizat2020implicit,ji2020directional}.
These theoretical results all assume the stepsize of GD is small (and even infinitesimal) such that the empirical risk decreases monotonically and, therefore cannot be directly applied to large stepsize GD used in practice.

More recently, large stepsize GD that induces oscillatory risk has been analyzed in simplified setups \citep[see][for an incomplete list of references]{ahn2023learning,zhu2022understanding,kreisler2023gradient,chen2023beyond,wang2022large,wu2023implicit,wu2024large}.
In particular, in logistic regression with linearly separable data, \citet{wu2023implicit} showed that the implicit bias of GD (that maximizes the margin) holds not only for small stepsizes \citep{soudry2018implicit,ji2018risk} but also for an arbitrarily large stepsize.
In the same problem, \citet{wu2024large} further showed that large stepsize GD that undergoes risk oscillation can achieve an $\tilde\Ocal(1/t^2)$ empirical risk, whereas small stepsize GD that monotonically decreases the empirical risk must suffer from a $\Omega(1/t)$ empirical risk. 
Nonetheless, these theories of large stepsize GD are limited to relatively simple setups such as linear models. The theory of large stepsize GD for non-linear networks is underdeveloped.



This work fills the gap by providing an analysis of large stepsize GD for non-linear networks. In the following, we set up our problem formally and summarize our contributions. 

\begin{figure}[t!]
\centering
\subfigure[Empirical risk]{
    \includegraphics[width=.3\linewidth]{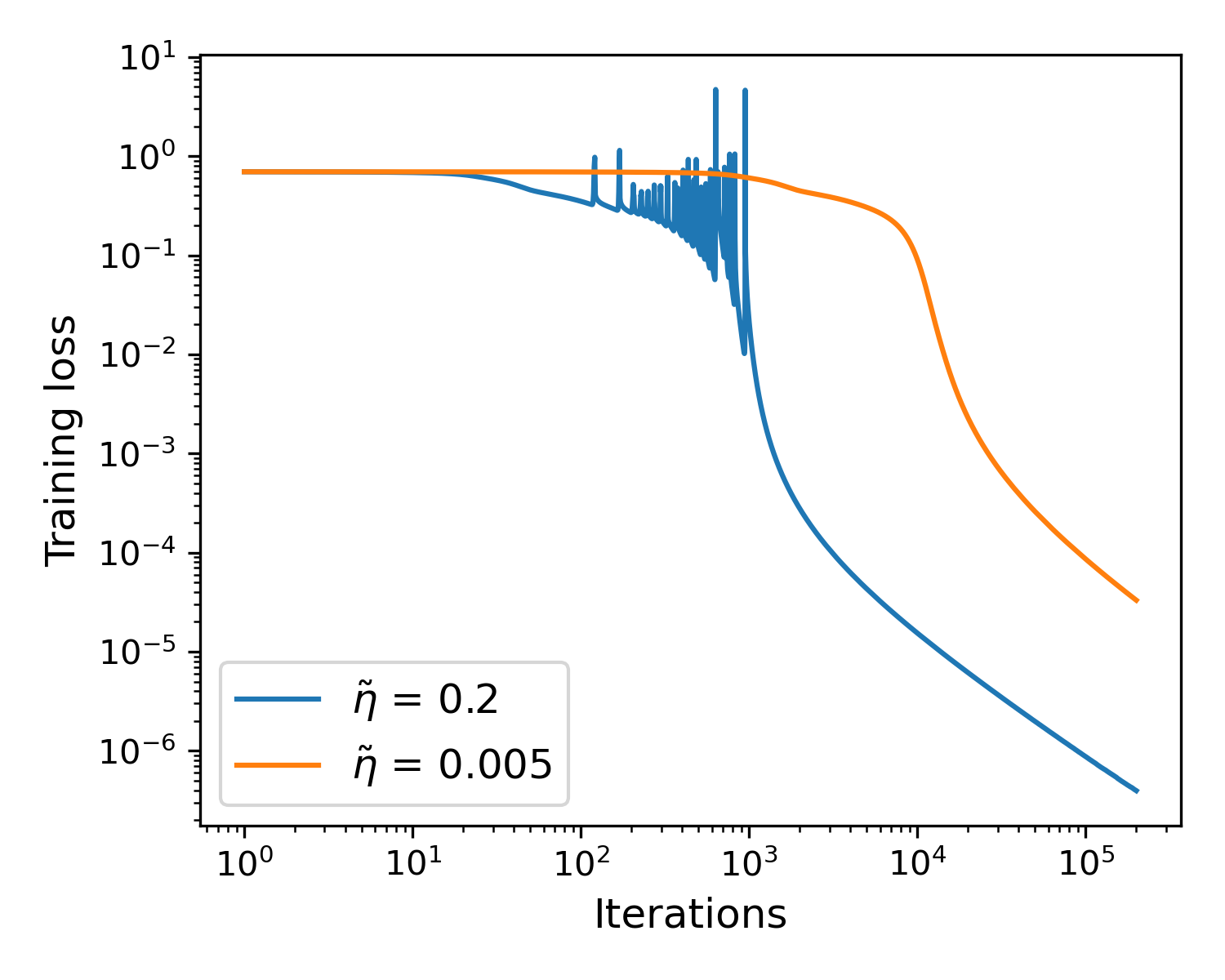}
    \label{fig:sfig1-loss}
}
\hfill 
\subfigure[Normalized margin]{
    \includegraphics[width=.3\linewidth]{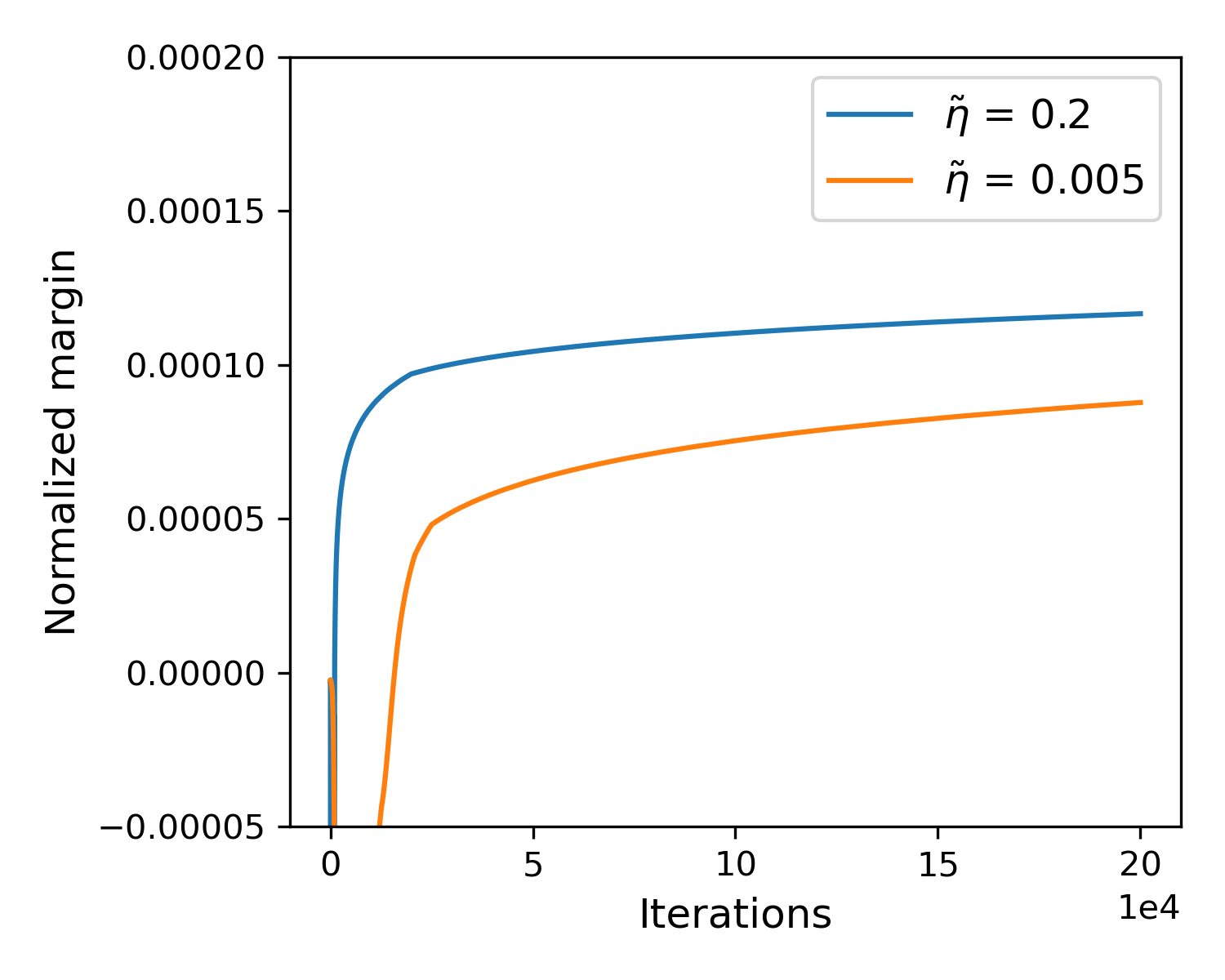}
    \label{fig:sfig1-margin}
}
\hfill 
\subfigure[Test accuracy]{
    \includegraphics[width=.3\linewidth]{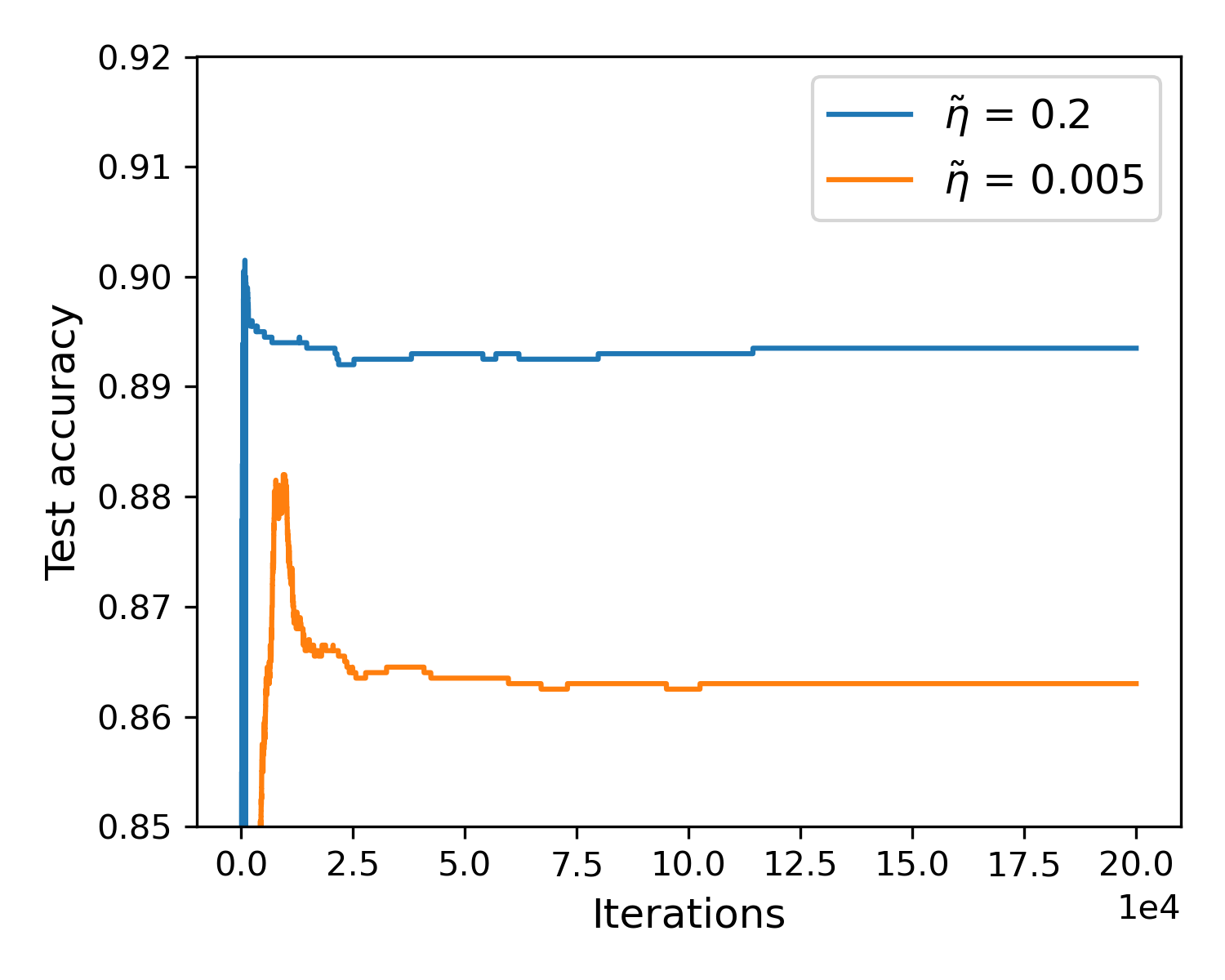}
    \label{fig:sfig1-acc}
}
\caption{The behavior of \eqref{eq: GD} for optimizing a non-homogenous four-layer MLP with GELU activation function on a subset of CIFAR-10 dataset. We randomly sample $6,000$ data with labels ``airplane'' and ``automobile'' from CIFAR-10 dataset. 
The normalized margin is defined as $\big(\arg\min_{i\in [n]} y_if (\wB_t;\xB_i)\big) / \|\wB_t\|^4$, which is close to  \Cref{eq: norm margin}.
The blue curves correspond to GD with a large stepsize $\tilde \eta=0.2$, where the empirical risk oscillates in the first phase but decreases monotonically in the second phase. 
The orange curves correspond to GD with a small stepsize $\tilde \eta =0.005$, where the empirical risk decreases monotonically.
Furthermore, \Cref{fig:sfig1-margin} suggests the normalized margins of both two curves increase and converge in the stable phases. 
Finally, \Cref{fig:sfig1-acc} suggests that large stepsize achieves a better test accuracy, consistent with larger-scale learning experiment \citep{hoffer2017train,goyal2017accurate}. 
More details can be found in \Cref{sec:experiments}.
}
\label{fig:sfig1}
\end{figure}

\paragraph{Setup.}
Consider a binary classification dataset $(\xB_i,y_i)_{i=1}^n$, where $\xB_i\in \Rbb^d$ is a feature vector and $y_i\in \{\pm 1\}$ is a binary label. For simplicity, we assume $\|\xB_i\|\le 1$ for all $i$ throughout the paper. 
For a predictor $f$, the empirical risk under logistic loss is defined as
\begin{equation}
\label{eq: logistic loss}
    L(\wB) := \frac{1}{n}\sum_{i=1}^n \ell( y_i f(\wB; \xB_i)), \quad \ell(t) := \log(1+e^{-t}).
\end{equation}
Here, the predictor $f(\wB; \cdot): \Rbb^d \mapsto \Rbb$ is parameterized by trainable parameters $\wB$ and is assumed to be continuously differentiable for $\wB$.
The predictor is initialized from $\wB_0$ and then trained by \emph{gradient descent} (GD) with a constant stepsize $\tilde\eta >0$, that is,
\begin{equation}
\label{eq: GD}
    \wB_{t+1} := \wB_{t} - \tilde \eta  \nabla L(\wB_t),\quad t\ge 0. \tag{\text{GD}}
\end{equation}
We are interested in a non-linear predictor $f$ and a large stepsize $\tilde\eta$. A notable example in our theory is two-layer networks with Lipschitz, smooth, and nearly homogenous activations (see \Cref{eq: 2nn}). Note that minimizing $L(\wB)$ is a non-convex problem in general. 

\paragraph{Observation.}
Empirically, large stepsize GD often undergoes a phase transition, where the empirical risk defined in \Cref{eq: logistic loss} oscillates in the first phase but decreases monotonically in the second phase (see empirical evidence in Appendix A in \citep{cohen2020gradient} and a formal proof in \citep{wu2024large} for linear predictors). This is illustrated in \Cref{fig:sfig1} (the experimental setup is described in \Cref{sec:experiments}). We follow \citet{wu2024large} and call the two phases the \emph{edge of stability} (EoS) phase \citep[name coined by][]{cohen2020gradient} and the \emph{stable} phase, respectively. 

\paragraph{Contributions.}
We prove the following results for large stepsize GD for training non-linear predictors under logistic loss.

\begin{enumerate}[leftmargin=*]
    \item 
    For Lipschitz and smooth predictor $f$ trained by GD with stepsize $\tilde\eta$, we show that as long as the empirical risk is below a threshold depending on $\tilde\eta$, GD monotonically decreases the empirical risk (see \Cref{thm: Implicit bias and bound of stable phase}). This result extends the stable phase result in \citet{wu2024large} from linear predictors to non-linear predictors, demonstrating the generality of the existence of a stable phase.

    \item Assuming that GD enters the stable phase, if in addition the preditor has a bounded homogenous error (see \Cref{assump:model:near-homogeneous}), we show that the normalized margin induced by GD, $\min_i y_i f(\wB_t;\xB_i)/ \|\wB_t\|$, nearly monotonically increases (see \Cref{thm: Implicit bias and bound of stable phase}). 
    This characterizes an implicit bias of GD for \emph{non-homogenous} predictors. 
    In particular, our theory covers two-layer networks with non-homogenous activations such as GeLU and SiLU that cannot be covered by existing results \citep{lyu2020gradient,ji2020directional,chizat2020implicit}.
    

    \item Under additional technical assumptions (the dataset is linearly separable and the derivative of the activation function is bounded away from zero), we show that the initial EoS phase must stop in $\Ocal(\tilde\eta)$ steps and GD transits to the stable phase afterwards. Furthermore, by choosing a suitably large stepsize, GD achieves a $\tilde\Ocal(1/t^2)$ empirical risk after $t$ steps. In comparison, GD that converges monotonically incurs an $\Omega(1/t)$ risk. This result indicates an optimization benefit of using large stepsize and generalizes the results in \citep{wu2024large} from linear predictors to neural networks.

    
    

\end{enumerate}


\section{Stable Phase and Margin Improvement}\label{sec:stable}
In this section, we present our results for the stable phase of large stepsize GD in training non-linear predictors.
Specifically, our results apply to non-linear predictors that are Lipschitz, smooth, and nearly homogeneous, as described by the following assumption.

\begin{assumption}[Model conditions]\label{assump:model}
Consider a predictor $f(\wB;\xB_i)$, where $\xB_i$ is one of the feature vectors in the training set.   
\begin{assumpenum}
    \item \label{assump:model:bounded-grad} \textbf{Lipschitzness}. Assume there exists $\rho>0$  such that for every $\wB$, $\sup_i \|\nabla_\wB f(\wB;\xB_i)\| \le \rho$.
    \item \label{assump:model:smooth} \textbf{Smoothness}.
    Assume there exists $\beta>0$ such that for all $\wB, \vB$, 
    \[ \|\grad f(\wB;\xB_i) - \grad f(\vB;\xB_i)\| \le \beta \|\wB -\vB\|,\quad i=1,\dots,n.\]
    \item \label{assump:model:near-homogeneous} \textbf{Near-homogeneity}. Assume there exists $\kappa>0$ such that for every $\wB$, 
    \[
        |f(\wB;\xB_i) - \langle \nabla_\wB f(\wB; \xB_i), \wB\rangle | \le \kappa,\quad i=1,\dots,n.
    \]
\end{assumpenum}
\end{assumption}
\Cref{assump:model:bounded-grad,assump:model:smooth} are commonly used conditions in the optimization literature. 
If $\kappa=0$, then \Cref{assump:model:near-homogeneous} requires the predictor to be exactly $1$-homogenous. Our \Cref{assump:model:near-homogeneous} allows the predictor to have a bounded homogenous error. 
It is clear that linear predictors $f(\wB;\xB) : = \wB^\top \xB$ satisfy \Cref{assump:model} with $\rho=\sup_{i}\|\xB_i\|\le 1$, $\beta=0$, and $\kappa = 0$.
Another notable example is \emph{two-layer networks} given by
\begin{equation}
   \label{eq: 2nn}
    f(\wB; \xB) := \frac{1}{m}\sum_{j=1}^m a_j \phi(\xB^\top \wB^{(j)}),  \quad \wB^{(j)} \in \Rbb^d, \quad j=1,\dots,m,
\end{equation}
where we assume $a_j \in \{\pm 1\}$ are fixed and $\wB = (\wB^{(j)} )_{j=1}^m \in \Rbb^{md}$ are the trainable parameters. 
We define two-layer networks with the mean-field scaling \citep{song2018mean,chizat2020implicit}. However, our results hold for any width.
The effect of rescaling the model will be discussed in \Cref{sec:transition}.
The following example shows that \Cref{assump:model} covers two-layer networks with many commonly used activations $\phi(\cdot)$. The proof is provided in \Cref{sec: eg homo act}.


\begin{example}[Two-layer networks]
    \label{eg: homo act}
    Two-layer networks defined in \eqref{eq: 2nn} with the following activation functions satisfy \Cref{assump:model} with the described constants:  
    \begin{itemize}[leftmargin=*]
         \item \textbf{GELU}. $\phi(x) :=  \frac{x}{2}\erf \big( 1+(x/ \sqrt{2})\big) $ with $\kappa= {e^{-1/2}}/{\sqrt{2\pi}}$, $\beta=2/m$, and $\rho =  (\sqrt{2\pi}+ e^{-1/2})/\sqrt{2\pi m}$.
    \item \textbf{Softplus}. $\phi(x) := \log(1+e^x)$ with $\kappa=\log{2}$, $\beta=1/m$, and $\rho = 1/\sqrt{m}$.
    \item \textbf{SiLU}. $\phi(x) := {x}/{(1+e^{-x})}$  with $\kappa=1$, $ \beta=4/m$, and $\rho =  2/\sqrt{m}$.
    \item \textbf{Huberized ReLU} \citep{chatterji2021does}. For a fixed $h>0$,
    \[
    \phi(x) \coloneqq 
    \begin{cases}
        0 & x<0, \\ 
        \frac{x^2}{2 h} &  0\le x \le h, \\ 
        x-\frac{h}{2} &x>h,
    \end{cases}
    \]
    with $\kappa=h/2$, $\beta=1/(hm)$, and $\rho=1/\sqrt{m}$. 
    \end{itemize}
\end{example}

\paragraph{Margin for nearly homogenous predictors.}
For a nearly homogenous predictor $f(\wB;\cdot)$ (see \Cref{assump:model:near-homogeneous}), we define its \emph{normalized margin} (or \emph{margin} for simplicity) as
\begin{equation}
\label{eq: norm margin}
    \bar \gamma(\wB) := \frac{\min_{i\in[n]} y_i f(\wB;\xB_i)}{\|\wB\|}. 
\end{equation}
A large normalized margin $\bar \gamma(\wB)$ guarantees the prediction of each sample is away from the decision boundary. 
The normalized margin \Cref{eq: norm margin} is introduced by \citet{lyu2020gradient} for homogenous predictors.
However, we show that the same notion is also well-defined for non-homogenous predictors that satisfy \Cref{assump:model:near-homogeneous}.
The next theorem gives sufficient conditions for large stepsize GD to enter the stable phase in training non-homogenous predictors and characterizes the increase of the normalized margin. The proof of \Cref{thm: Implicit bias and bound of stable phase} is deferred to \Cref{sec:stable-phase}.

\begin{theorem}
[Stable phase and margin improvement]
\label{thm: Implicit bias and bound of stable phase}
Consider \Cref{eq: GD} with stepsize $\tilde\eta$ on a predictor $f(\wB;\xB)$ that satisfies \Cref{assump:model:bounded-grad,assump:model:smooth}.
If there exists $r \ge 0$ such that 
\begin{equation}
\label{eq: loss thresh1}
    L(\wB_r) \le  \frac{1}{\tilde \eta (2\rho^2 +  \beta)} ,
\end{equation}
then GD is in the stable phase for $t\ge r$, that is, $(L(\wB_t))_{t\ge r}$ decreases monotonically. 
If additionally the predictor satisfies \Cref{assump:model:near-homogeneous} and there exists $s \ge 0$ such that 
\begin{equation}
\label{eq: loss thresh2}
    L(\wB_s) \le \min \Big\{\frac{1}{e^{\kappa+2}2n},\ \frac{1}{\tilde \eta (4\rho^2 +  2\beta)}\Big\} ,
\end{equation}
we have the following for $t \ge s$:
\begin{itemize}[leftmargin=*]
\item \textbf{Risk convergence}. $L(\wB_t) = \Theta(1/t)$, where $\Theta$ hides the dependence on $L(\wB_s), \|\wB_s\|, \tilde \eta$ and $\rho$. 
 \item \textbf{Parameter increase}. $\|\wB_{t+1}\|  \ge  \|\wB_t\|$ and $\|\wB_t\| = \Theta(\log  (t))$, where $\Theta$ hides the dependence on $n$, $\rho$, $\kappa$, $\tilde \eta$,  $\|\wB_s\|$, $L(\wB_s)$ and $f(\wB_s;\xB_i) $ for $i=1,\dots,n$.
\item \textbf{Margin improvement}. There exists a modified margin function $\gamma^c(\wB)$ such that 
\begin{itemize}
    \item $\gamma^c(\wB_t)$ is increasing and bounded.
    \item $\gamma^c(\wB_t)$ is a multiplicative approximiator of $\bar \gamma(\wB_t)$, that is,
    there exists $c>0$ such that 
    \[
    \gamma^c(\wB_t) \le  \bar \gamma(\wB_t) \le \bigg(1+\frac{c}{\log (1/{L(\wB_t))}} \bigg) \gamma^c(\wB_t),\quad t\ge s.
    \]
\end{itemize}
As a direct consequence, $\lim_{t\to\infty} \bar \gamma(\wB_t) =\lim_{t\to\infty} \gamma^c(\wB_t) $.
\end{itemize}
\end{theorem}
\Cref{thm: Implicit bias and bound of stable phase} shows that for an arbitrarily large stepsize $\tilde\eta$, GD must enter the stable phase if the empirical risk falls below a threshold depending on $\tilde\eta$ given by \Cref{eq: loss thresh1}. 
Furthermore, for nearly homogenous predictors, \Cref{thm: Implicit bias and bound of stable phase} shows that under a stronger risk threshold condition \Cref{eq: loss thresh2}, the risk must converge at a $\Theta(1/t)$ rate and that the normalized margin nearly monotonically increases. 
This demonstrates an implicit bias of GD, even when used with a large stepsize and the trained predictor is non-homogenous. 

\paragraph{Limitations.}
The stable phase conditions in \Cref{thm: Implicit bias and bound of stable phase} require GD to enter a sublevel set of the empirical risk. However, such a sublevel set might be empty. 
For instance, let $f(\wB; \xB)$ be a two-layer network \Cref{eq: 2nn} with sigmoid activation. 
Notice that the predictor is uniformly bounded, $|f(\wB; \xB)| \le 1$, so we have
\[
L(\wB) = \frac{1}{n} \sum_{i=1}^n \log(1+e^{-y_i f(\wB;\xB_i)}) 
\ge \log(1+e^{-1}).
\]
On the other hand, we can also verify that \Cref{assump:model:near-homogeneous} is satisfied by $f(\wB;\xB)$ with $\kappa =1$ but no smaller $\kappa$.
Therefore \Cref{eq: loss thresh2} cannot be satisfied.
In general, the sublevel set given by the right-hand side of \Cref{assump:model:near-homogeneous} is non-empty if 
\begin{align*}
    \text{there exists a unit vector $\vB$ such that} \ \min_i y_i f(\lambda \vB; \xB_i) \to \infty \ \text{as}\ \lambda \to \infty.
\end{align*}
The above condition requires the data can be separated arbitrarily well by some predictor within the hypothesis class.
This condition is general and covers (sufficiently large) two-layer networks \Cref{eq: 2nn} with many commonly used activations such as GeLU and SiLU. Moreover, although two-layer networks with sigmoid activation violate this condition, they can be modified by adding a leakage to the sigmoid to satisfy the condition. 

In the next section, we will provide sufficient conditions such that large stepsize GD will enter the stable phases characterized by \Cref{eq: loss thresh1} or \Cref{eq: loss thresh2}.

\paragraph{Comparisons to existing works.}
Our \Cref{thm: Implicit bias and bound of stable phase} makes several important extensions compared to existing results. 
First, \Cref{thm: Implicit bias and bound of stable phase} suggests that the stable phase happens for general non-linear predictors such as two-layer networks, while the work by \citet{wu2024large} only studied the stable phase for linear predictors. 
Second, the margin improvement is only known for small (and even infinitesimal) stepsize GD and homogenous predictors \citep{lyu2020gradient,chizat2020implicit,ji2020directional}.
To the best of our knowledge, \Cref{thm: Implicit bias and bound of stable phase} is the first implicit bias result that covers large stepsize GD and non-homogenous predictors.

From a technical perspective, our proof uses techniques introduced by \citet{lyu2020gradient} for analyzing homogenous predictors. 
Our main innovations are constructing new auxiliary margin functions that can deal with errors caused by large stepsize and non-homogeneity. 
More details are discussed in \Cref{sec:margin}.




\section{Edge of Stability Phase}\label{sec:eos}
Our stable phase results in \Cref{thm: Implicit bias and bound of stable phase} require the risk to be below a certain threshold (see \Cref{eq: loss thresh1,eq: loss thresh2}).
In this section, we show that the risk can indeed be below the required threshold, even when GD is used with large stepsize. 
Recall that minimizing the empirical risk with a non-linear predictor is non-convex, therefore solving it by GD is hard in general. 
We make additional technical assumptions to conquer the challenges caused by non-convexity. We conjecture that these technical assumptions are not necessary and can be relaxed. 

We focus on two-layer networks \Cref{eq: 2nn}. We make the following assumptions on the activation function.


\begin{assumption}
[Activation function conditions] 
\label{assump: activation}
In the two-layer network \Cref{eq: 2nn}, let the activation function $\phi:\Rbb \to \Rbb$ be continuously differentiable. Moreover,
\begin{assumpenum}
    \item \label{assump: activation:grad} \textbf{Derivative condition}. Assume there exists $0<\alpha<1$ such that $\alpha \le | \phi^{\prime}(z)| \le 1$.
\item \label{assump:activation:smooth} \textbf{Smoothness}.
    Assume there exists $\tilde\beta>0$ such that for all $x,y\in\Rbb$, \( |\phi^{\prime}(x) - \phi^{\prime}(y)| \le \tilde\beta |x-y|.\)
    \item \label{assump: activation:near-homogeneous} \textbf{Near-homogeneity}. Assume there exists $\kappa>0$ such that for every $z\in \Rbb$, $|\phi(z) - \phi^\prime (z)z | \le \kappa$.
\end{assumpenum}
\end{assumption}

Recall that $\sup_i\|\xB_i\|\le 1$.
One can then check by direct computation that, under \Cref{assump: activation}, two-layer networks \Cref{eq: 2nn} satisfy \Cref{assump:model} with $\rho=1/\sqrt{m}$, $\beta=\tilde\beta/m$, and $\kappa=\kappa$.

\Cref{assump:activation:smooth,assump: activation:near-homogeneous} cover many commonly used activation functions. 
In \Cref{assump: activation:grad}, we assume $|\phi'(z)| \le 1$. This is just for the simplicity of presentation and our results can be easily generalized to allow $|\phi'(z)| \le C$ for a constant $C>0$.
The other condition in \Cref{assump: activation:grad}, $|\phi'(z)| \ge \alpha$, however, is non-trivial.
This condition is widely used in literature \citep[see][and references thereafter]{brutzkus2018sgd,frei2021provable} to facilitate GD analysis.
Technically, this condition guarantees that each neuron in the two-layer network \Cref{eq: 2nn} will always receive a non-trivial gradient in the GD update; otherwise, neurons may be frozen during the GD update. 
Furthermore, commonly used activation functions can be combined with an identity map to satisfy \Cref{assump: activation:grad}.
This is formalized in the following example. The proof is provided in \Cref{sec:eg:leak-homo}.


\begin{example}
    [Leaky activation functions]
    \label{eg: leak homo act}
    Fix $0.5\le c<1$. 
    \begin{itemize}[leftmargin=*]
    \item Let $\phi$ be GELU, Softplus, or SilU in \Cref{eg: homo act}, then its modification $\tilde\phi(x) := c x + (1-c)/4 \cdot  \phi(x)$ satisfies \Cref{assump: activation} with $\kappa=1$, $\alpha=0.25$, and $\tilde\beta= 1$. 
    In particular, the modification of softplus can be viewed as a \emph{smoothed} leaky ReLU. 
    \item Let $\phi$ be the Huberized ReLU in \Cref{eg: homo act}, then its modification $\tilde\phi(x) := c x + (1-c)/4 \cdot  \phi(x)$ satisfies \Cref{assump: activation} with $ \kappa=h/2$, $\alpha=0.5$, and $\tilde \beta = 1/4h$. 
        \item The ``leaky'' tanh, $\tilde \phi(x) \coloneqq c x + (1- c) \tanh(x)$, and the ``leaky'' sigmoid, $\tilde \phi(x) \coloneqq cx + {c}/{(1+e^{-x})}$,  both satisfy \Cref{assump: activation} with $\kappa=1, \alpha = 0.5$ and $\tilde\beta = 1$. 
    \end{itemize}
    \end{example}

For the technical difficulty of non-convex optimization, 
we also need to assume a linearly separable dataset to conduct our EoS phase analysis.
\begin{assumption}[Linear separability] \label{assump:separable}
    Assume there is a margin $\gamma >0$ and a unit vector $\wB_*$ such that $ y_i\xB_i^\top \wB_* \ge \gamma$ for every $i=1,\ldots ,n$.  
\end{assumption}

The following theorem shows that when GD is used with large stepsizes, the average risk must decrease even though the risk may oscillate locally. The proof of \Cref{thm: Bound of EoS phase} is defered to \Cref{sec:eos}. 
\begin{theorem}
    [The EoS phase for two-layer networks]
    \label{thm: Bound of EoS phase}
    Under \Cref{assump:separable}, consider \Cref{eq: GD} on two-layer networks \Cref{eq: 2nn} that satisfy \Cref{assump: activation:grad,assump: activation:near-homogeneous}. Denote the stepsize by $\tilde \eta := m \eta $, where $m$ is the network width and $\eta$ can be arbitrarily large. 
    Then for every $t>0$, we have 
    \begin{equation*}
    \frac{1}{t} \sum_{k=0}^{t-1} L (\wB_k )  \leq \frac{1+8\log ^2(\gamma^2 \eta t)/\alpha ^2+ 8\kappa^2/\alpha ^2+\eta^2 }{\gamma^2 \eta t} + \frac{\|\wB_0\|^2}{ m\eta t} = \Ocal\bigg(\frac{\log^2(\eta t) + \eta^2}{\eta t} \bigg).
    \end{equation*}
\end{theorem}

\Cref{thm: Bound of EoS phase} suggests that the average risk of training two-layer networks decreases even when GD is used with large stepsize. Consequently, the risk thresholds \Cref{eq: loss thresh1,eq: loss thresh2} for GD to enter the stable phase must be satisfied after a finite number of steps. This will be discussed in depth in the next section.

Compared to Theorem 1 in \citep{wu2024large},
\Cref{thm: Bound of EoS phase} extends their EoS phase bound from linear predictors to two-layer networks.

\section{Phase Transition and Fast Optimization}\label{sec:transition}

For two-layer networks trained by large stepsize GD,  \Cref{thm: Bound of EoS phase} shows that the average risk must decrease over time. 
Combining this with \Cref{thm: Implicit bias and bound of stable phase}, GD must enter the stable phase in finite steps, and the loss must converge while the normalized margin must improve. 

However, a direct application of \Cref{thm: Bound of EoS phase} only leads to a suboptimal bound on the phase transition time. 
Motivated by \citet{wu2024large}, we establish the following sharp bound on the phase transition time by tracking the gradient potential (see \Cref{lem: Gradient potential bound in EoS}). The proof of \Cref{thm:phasetran} is deferred to \Cref{sec:phase}.

\begin{theorem} [Phase transition and stable phase for two-layer networks]
\label{thm:phasetran}
Under \Cref{assump:separable}, consider \Cref{eq: GD} on two-layer networks \Cref{eq: 2nn} that satisfy \Cref{assump: activation}. 
Clearly, the two-layer networks also satisfy \Cref{assump:model} with $\rho=1/\sqrt{m}$, $\beta=\tilde\beta/m$, and $\kappa=\kappa$.
Denote the stepsize by $\tilde \eta :=  m \eta$, where $m$ is the network width and $\eta>0$ can be arbitrarily large. 
\begin{itemize}[leftmargin=*]
    \item \textbf{Phase transition time}. There exists $s\le \tau$ such that \eqref{eq: loss thresh2} in \Cref{thm: Implicit bias and bound of stable phase} holds, where 
\[
    \tau := \frac{128(1+4\kappa)}{ \alpha ^2} \max \bigg \{ c_1\eta, c_2n, e, \frac{c_2\eta+c_1n}{\eta} \log \frac{c_2\eta+c_1n}{\eta}, \frac{(c_2\eta+c_1n)}{\eta}\cdot \frac{\|\wB_0\|}{\sqrt{m}}\bigg \},
\]
where $c_1 := 4e^{\kappa+2}$ and $c_2 := (8+4\tilde \beta)$. Therefore \Cref{eq: GD} is in the stable phase from $s$ onwards. 
\item \textbf{Explicit risk bound in the stable phase}. We have $(L(\wB_t))_{t\ge s}$ monotonically decreases and 
\[
        L(\wB_t) \le \frac{2}{\alpha ^2 \gamma ^2 \eta (t-s)},\quad t\ge s.
    \]
\end{itemize}
\end{theorem}

\Cref{thm: Implicit bias and bound of stable phase,thm: Bound of EoS phase,thm:phasetran} together characterize the behaviors of large stepsize GD in training two-layer networks. 
Specifically, large stepsize GD may induce an oscillatory risk in the beginning; but the averaged empirical risk must decrease (\Cref{thm: Bound of EoS phase}). After the empirical risk falls below a certain stepsize-dependent threshold, GD enters the stable phase, where the risk decreases monotonically (\Cref{thm:phasetran}). 
Finally, the normalized margin \Cref{eq: norm margin} induced by GD increases nearly monotonically as GD stays in the stable phase (\Cref{thm: Implicit bias and bound of stable phase}).


\paragraph{Fast optimization.}
Our bounds for two-layer networks are comparable to those for linear predictors shown by \citet{wu2024large}. 
Specifically, when used with a larger stepsize, GD achieves a faster optimization in the stable phase but stays longer in the EoS phase. 
Choosing a suitably large stepsize that balances the steps in EoS and stable phases, we obtain an \emph{accelerated} empirical risk in the following corollary. The proof is included in \Cref{sec:cor-acc}. 

\begin{corollary}
[Acceleration of large stepsize]
\label{cor:acceleration}
Under the same setup as in \Cref{thm:phasetran}, consider \eqref{eq: GD} with a given budget of $T$ steps such that
\[
    T \ge \frac{256(1+4\kappa)}{\alpha ^2 \gamma ^2} \max\bigg\{c_1 n,\  4c_2^2,\ \frac{2c_2\|\wB_0\|}{\sqrt{m} }\bigg \}, 
\]
where $c_1 := 4e^{\kappa+2}$ and $c_2 := (8+4\tilde \beta)$ are as in \Cref{thm:phasetran}.
Then for stepsize $\tilde\eta := \eta m$, where
\[
    \eta := \frac{\alpha ^2 \gamma ^2}{256(1+4\kappa)c_2} T, 
\]
we have $\tau \le T/2$ and
\[
    L(\wB_T)  \le \frac{2048(1+4\kappa) c_2}{\alpha ^4 \gamma^4} \cdot \frac{1}{T^2} = \Ocal \bigg(\frac{1}{T^2} \bigg). 
\]
\end{corollary}
\Cref{thm:phasetran} and \Cref{cor:acceleration} extend Theorem 1 and Corollary 2 in \citet{wu2024large} from linear predictors to two-layer networks. 
Another notable difference is that we obtain a sharper stable phase bound (and thus a better acceleration bound) compared to theirs, where we remove a logarithmic factor through a more careful analysis.

\Cref{cor:acceleration} suggests an accelerated risk bound of $\Ocal(1/T^2)$ by choosing a large stepsize that balances EoS and stable phases. We also show the following lower bound, showing that such acceleration is impossible if \Cref{eq: GD} does not enter the EoS phase. The proof is included in \Cref{sec:thm:classical}.

\begin{theorem}
[Lower bound in the classical regime]
\label{thm: Lower bound in the classical regime}
Consider \eqref{eq: GD} with initialization $\mathbf{w}_0 =0$ and stepsize $\tilde\eta>0$ for a two-layer network \Cref{eq: 2nn} satisfying \Cref{assump: activation}. 
Suppose the training set is given by 
$$
\mathbf{x}_1=(\gamma, \sqrt{1-\gamma^2}), \quad \mathbf{x}_2=(\gamma,-\sqrt{1-\gamma^2} / 2), \quad y_1=y_2=1, \quad 0<\gamma<0.1 .
$$

It is clear that $\left(\mathbf{x}_i, y_i\right)_{i=1,2}$ satisfy \Cref{assump:separable}. 
If $\left(L\left(\mathbf{w}_t\right)\right)_{t \geq 0}$ is non-increasing, then
$$
L\left(\mathbf{w}_t\right) \geq c_0 / t, \quad t \geq 1
$$
where $c_0>0$ is a function of $(\alpha, \phi, \xB_1, \xB_2, \gamma, \kappa, \beta)$ but is independent of $t$ and $\tilde \eta$.
\end{theorem}

\paragraph{Effect of model rescaling.}  
We conclude this section by discussing the impact of rescaling the model. Specifically, we replace the two-layer network in the mean-field scaling \Cref{eq: 2nn} by the following 
\[
    f(\wB;\xB) := b\cdot \frac{1}{m} \sum_{j=1}^m a_j \phi(\xB^\top \wB^{(j)}),
\]
and evaluate the impact of the scaling factor $b$ on our results. 
By choosing the optimal stepsize that balances the EoS and stable phases as in \Cref{cor:acceleration}, we optimize the risk bound obtained by GD with a fixed budget of $T$ steps and get the following bound. Detailed derivations are deferred to \Cref{sec:scaling}. 
\[
    L(\wB_T) = \begin{cases}
        \Ocal( 1/T^2) & \text{if } b \ge 1,\\
        \Ocal(b^{-3}/T^2 ) & \text{if } b < 1.
    \end{cases}
\]
This suggests that as long as $b\ge 1$, we get the same acceleration effect. In particular, the mean-field scaling $b=1$ \citep{song2018mean,chizat2020implicit} and the \emph{neural tangent kernel} (NTK) scaling $b=\sqrt{m}$ \citep{du2018gradient,jacot2018neural} give the same $\Ocal(1/ T^2)$ acceleration effect. 
An NTK analysis of large stepsize is included in \citep{wu2024large} and their conclusion is consistent with ours. 
Finally, we remark that our analysis holds for any width $m$ and uses techniques different from the mean-field or NTK methods. However, our acceleration analysis only allows linearly separable datasets.

\section{Experiments}\label{sec:experiments}
We conduct three sets of experiments to validate our theoretical insights. In the first set, we use a subset of the CIFAR-10 dataset \citep{krizhevsky2009learning}, which includes 6,000 randomly selected samples from the ``airplane'' and ``automobile'' classes. Our model is a multilayer perceptron (MLP) with four trainable layers and GELU activation functions, with a hidden dimension of 200 for each hidden layer. The MLP is trained using gradient descent with random initialization, as described in \Cref{eq: GD}. The results are shown in \Cref{fig:sfig1-loss,fig:sfig1-acc,fig:sfig1-margin}.

In the second set of experiments, we consider an XOR dataset consisting of four samples:
\begin{align*}
\xB_1=(-1,-1), y_1=1;\  \xB_2=(1,1), y_2=1;\ \xB_3=(1,-1),y_3=-1;\  \xB_4=(-1,1), y_4= -1.
\end{align*}
The above XOR dataset is not linearly separable. 
We test \Cref{eq: GD} with different stepsizes on a two-layer network \Cref{eq: 2nn} with the leaky softplus activation (see \Cref{eg: leak homo act} with $c=0.5$). 
The network width is $m=20$. The initialization is random.
The results are presented in \Cref{fig:sfig211-loss,fig:sfig212-rate,fig:sfig213-margin}.

In the third set of experiments, we consider the same task as in the first set of experiments, but we test \Cref{eq: GD} with different stepsizes on a two-layer network \Cref{eq: 2nn} with the softplus activation. 
The network width is $m=40$. The initialization is random.
The results are presented in \Cref{fig:sfig221-loss,fig:sfig222-rate,fig:sfig223-margin}.

\paragraph{Margin improvement.}
 \Cref{fig:sfig1-margin,fig:sfig213-margin,fig:sfig223-margin} show that the normalized margin nearly monotonically increases once gradient descent (GD) enters the stable phase, regardless of step size. This observation aligns with our theoretical findings in \Cref{thm: Implicit bias and bound of stable phase}.

\paragraph{Fast optimization.}
From \Cref{fig:sfig1-loss,fig:sfig211-loss,fig:sfig221-loss}, we observe that after GD enters the stable phase, a larger stepsize consistently leads to a smaller empirical risk compared to the smaller stepsizes, which is consistent with our \Cref{thm:phasetran} and \Cref{cor:acceleration}. 
Besides, \Cref{fig:sfig212-rate,fig:sfig222-rate} suggest that, asymptotically, GD converges at a rate of $\Ocal(1/(\tilde\eta t)) = \Ocal(1/(\eta t))$ (The width of networks is fixed), which verifies the sharpness of our stable phase bound in \Cref{thm:phasetran}.

\begin{figure}[t!]
\centering
\subfigure[Empirical risk, XOR.]{
    \includegraphics[width=.3\linewidth]{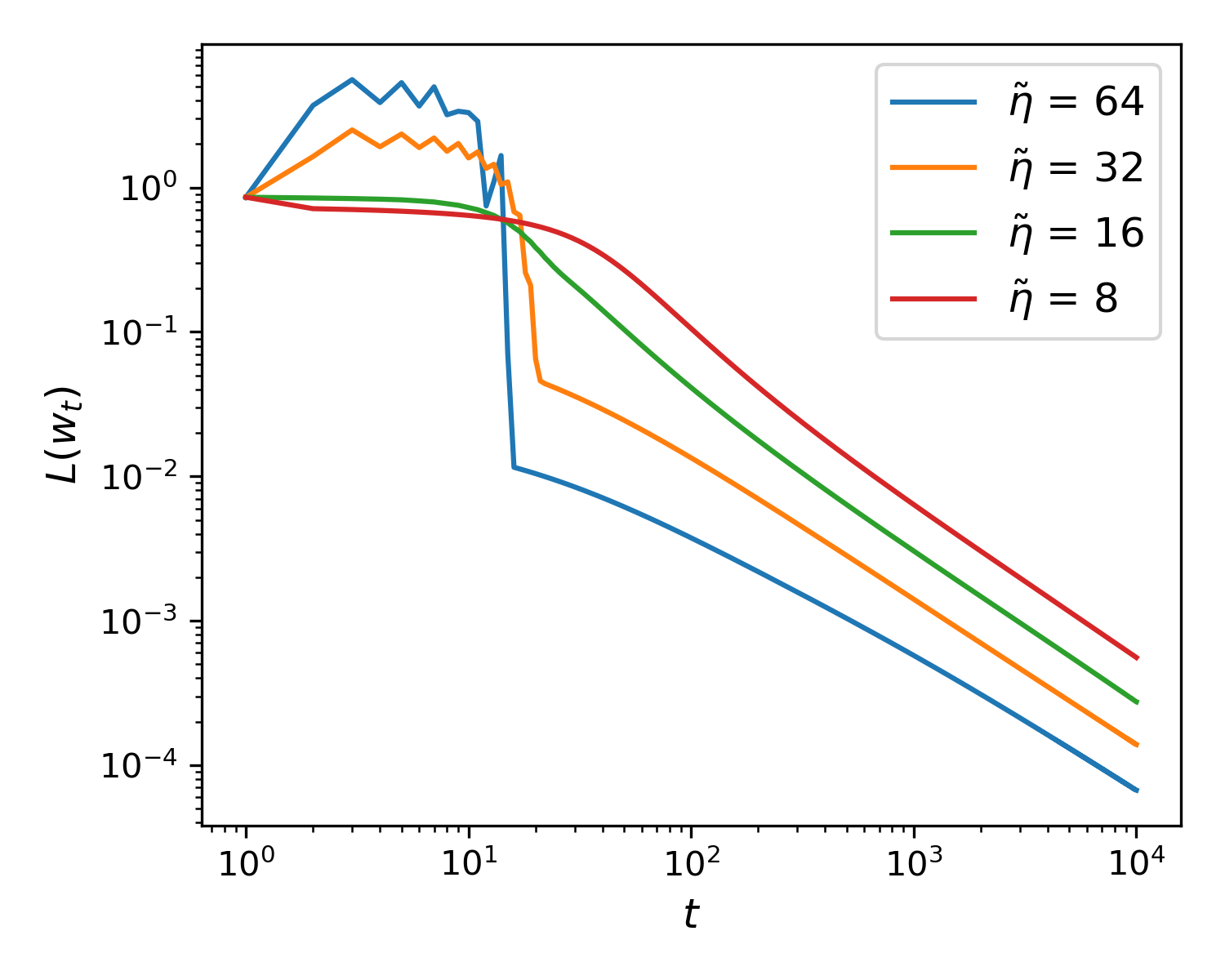}
    \label{fig:sfig211-loss}
}
\hfill 
\subfigure[Asymptotic rate, XOR.]{
    \includegraphics[width=.3\linewidth]{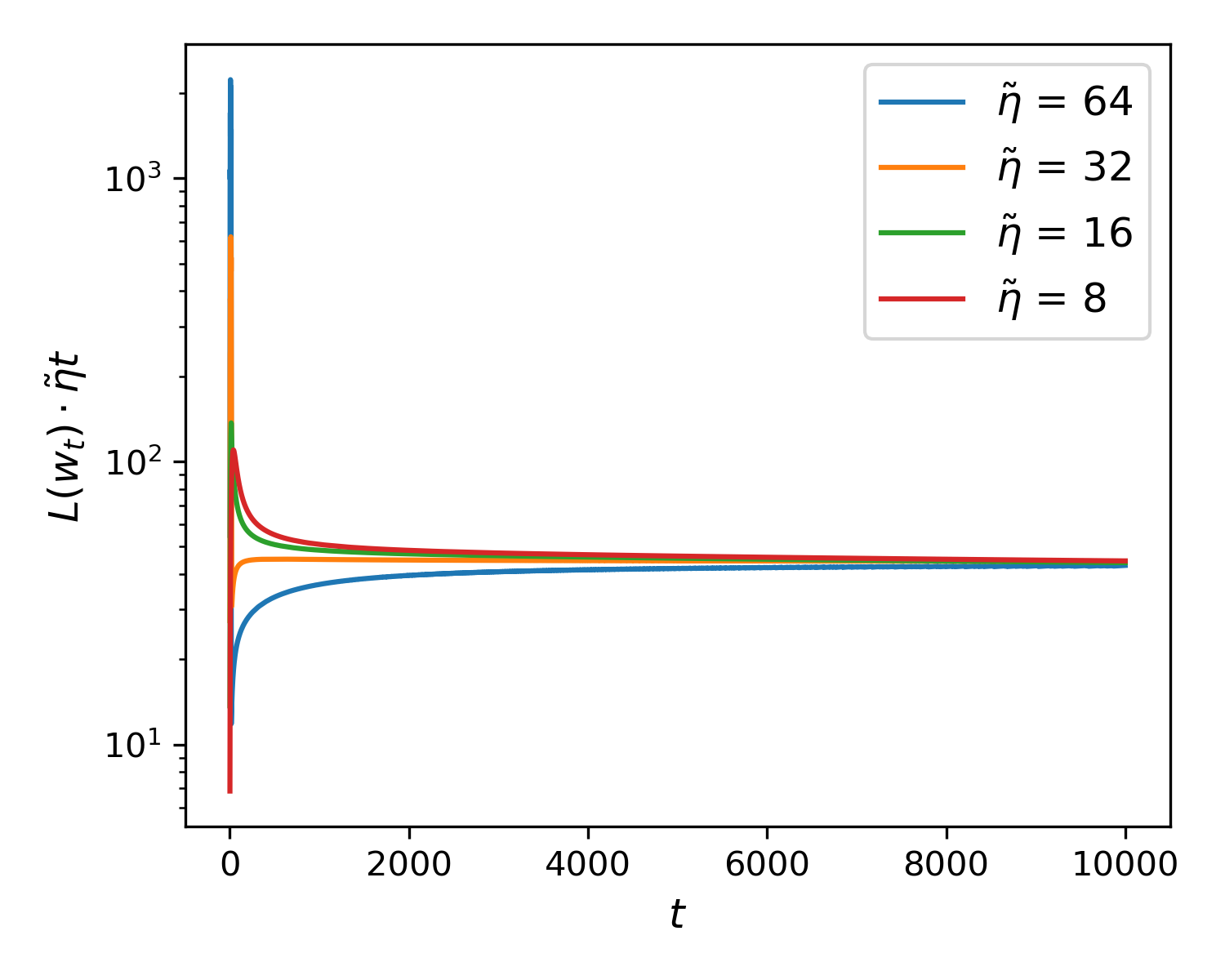}
    \label{fig:sfig212-rate}
}
\hfill 
\subfigure[Normalized margin, XOR. ]{
    \includegraphics[width=.3\linewidth]{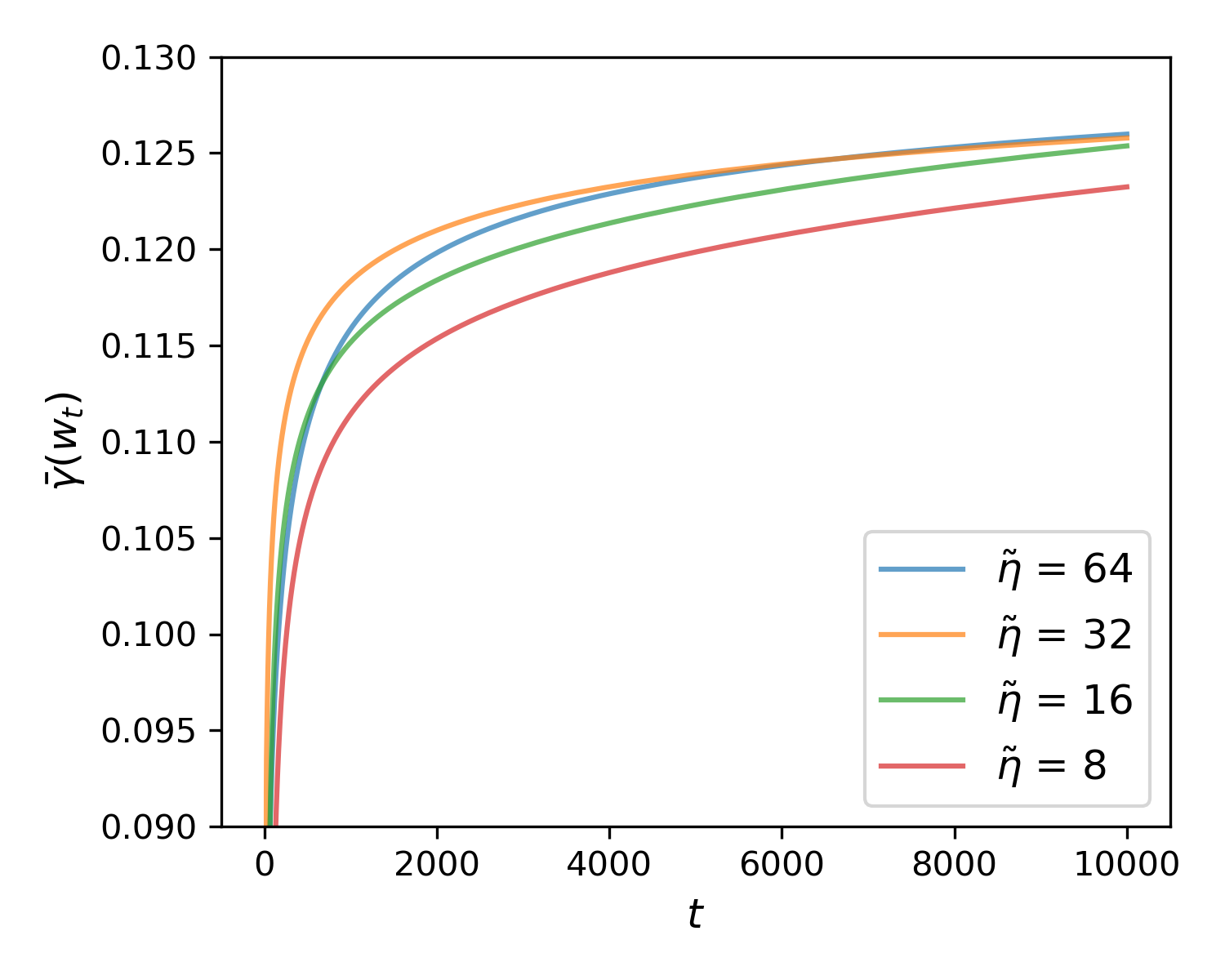}
    \label{fig:sfig213-margin}
}
\subfigure[Empirical risk, CIFAR-10.]{
    \includegraphics[width=.3\linewidth]{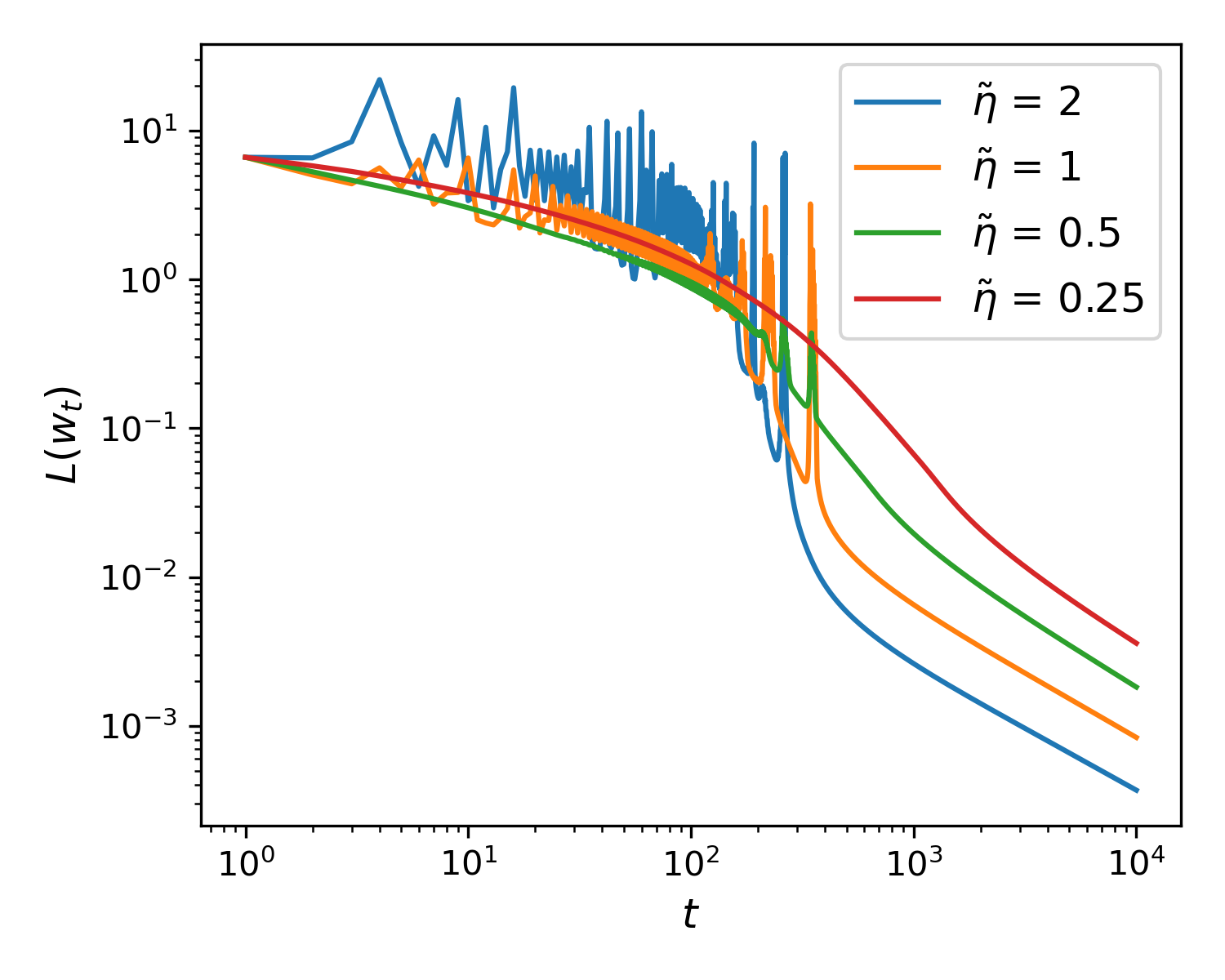}
    \label{fig:sfig221-loss}
}
\hfill 
\subfigure[Asymptotic rate, CIFAR-10.]{
    \includegraphics[width=.3\linewidth]{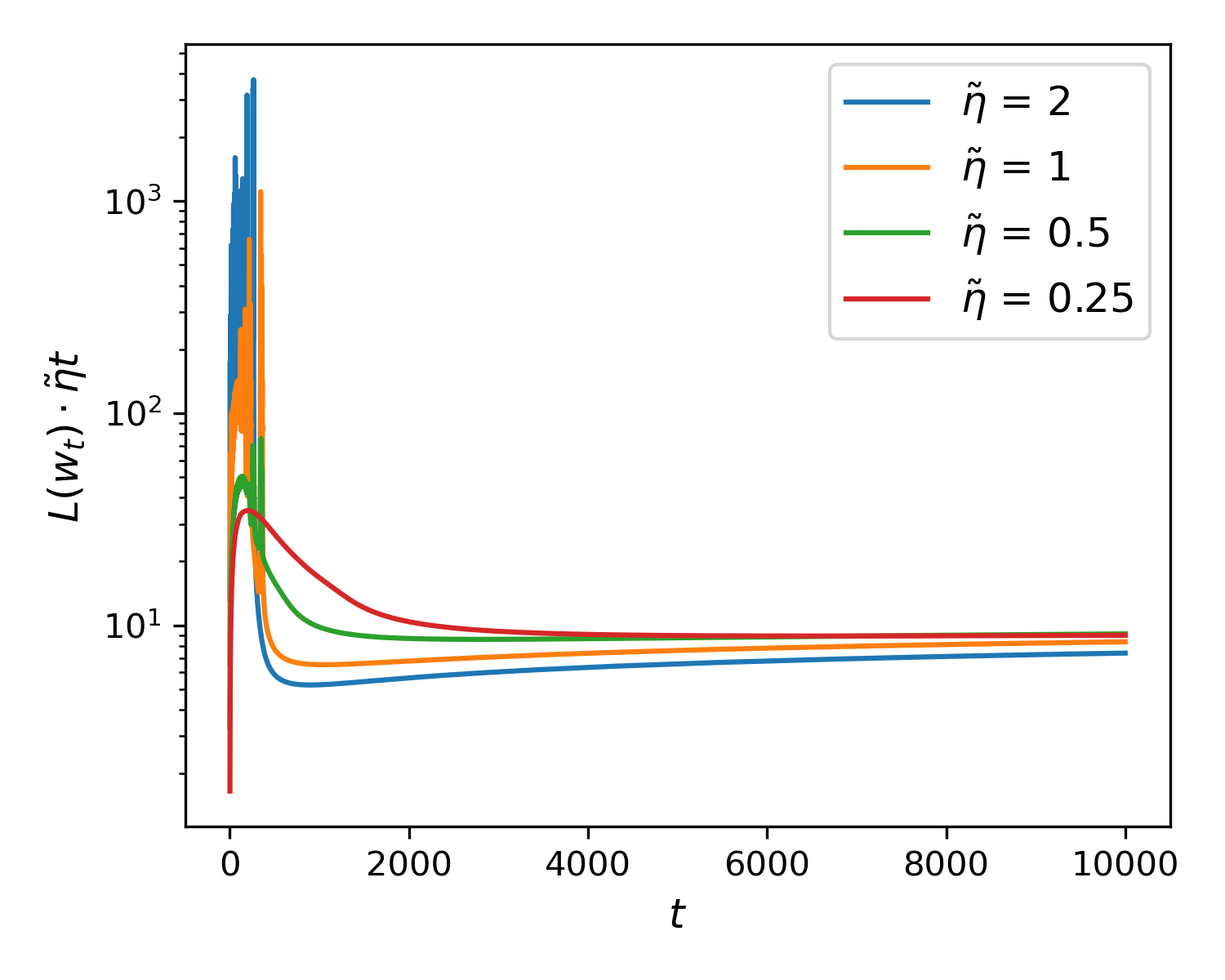}
    \label{fig:sfig222-rate}
}
\hfill 
\subfigure[Normalized margin, CIFAR-10.]{
    \includegraphics[width=.3\linewidth]{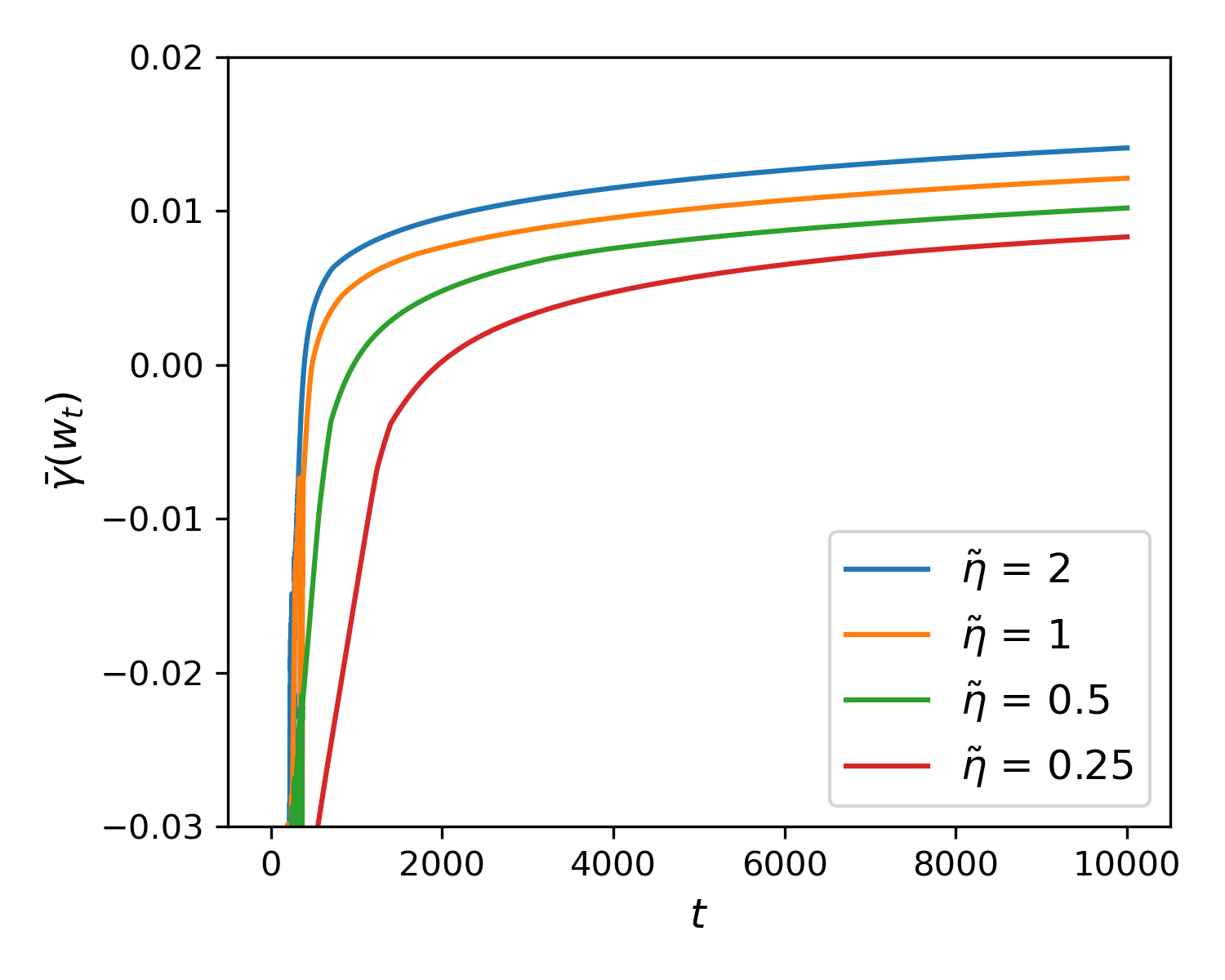}
    \label{fig:sfig223-margin}
}

\caption{Behavior of \eqref{eq: GD} for two-layer networks \Cref{eq: 2nn} with leaky softplus activation function (see \Cref{eg: leak homo act} with $c=0.5$).
We consider an XOR dataset and a subset of CIFAR-10 dataset. 
In both cases, we observe that (1) GD with a large stepsize achieves a faster optimization compared to GD with a small stepsize, (2) the asymptotic convergence rate of the empirical risk is $\Ocal(1/ (\tilde\eta t))$, and (3) in the stable phase, the normalized margin (nearly) monotonically increases.
These observations are consistent with our theoretical understanding of large stepsize GD.
More details about the experiments are explained in \Cref{sec:experiments}.
}
\label{fig:sfig2}
\end{figure}


\section{Related Works}\label{sec:related}
In this section, we discuss related papers.

\paragraph{Small stepsize and implicit bias.} 
For logistic regression on linearly separable data, \citet{soudry2018implicit,ji2018risk} showed that the direction of small stepsize GD converges to the max-margin solution. 
Their results were later extended by \citet{gunasekar2017implicit,gunasekar2018implicit,nacson2019stochastic,nacson2019lexicographic, nacson2019convergence,ji2021fast,lyu2020gradient,ji2020directional,chizat2020implicit,chatterji2021does,kunin2022asymmetric} to other algorithms and non-linear models.
However, in all of their analysis, the stepsize of GD needs to be small such that the empirical risk decreases monotonically. 
In contrast, our focus is GD with a large stepsize that induces non-monotonic risk. 

Two papers \citep{nacson2019lexicographic,kunin2022asymmetric} studied margin maximization theory for a special form of non-homogenous models. Specifically, when viewed in terms of different subsets of the trainable parameters, the model is homogeneous, although the order of homogeneity may vary. 
Compared to their setting, our non-homogenous models only require a bounded homogenous error (see \Cref{assump:model:near-homogeneous}). 
Therefore, our theory can cover two-layer networks \Cref{eq: 2nn} with non-homogeneous activations such as GELU and SiLU that cannot be covered by \citep{nacson2019lexicographic,kunin2022asymmetric}. 


\paragraph{Large stepsize and EoS.} 
In practice, large stepsizes are often preferred when using GD to train neural networks to achieve effective optimization and generalization performance \citep[see][and references therein]{wu2018sgd,cohen2020gradient}. 
In such scenarios, the empirical risk often oscillates in the beginning. This phenomenon is named \emph{edge of stability} (EoS) by \citet{cohen2020gradient}.
The theory of EoS is mainly studied in relatively simplified cases such as one- or two-dimensional functions \citep{zhu2022understanding, chen2023beyond, ahn2022understanding,kreisler2023gradient,wang2023good}, linear model \citep{wu2023implicit,wu2024large}, matrix factorization \citep{wang2022large,chen2023beyond}, scale-invariant networks \citep{lyu2022understanding}, for an incomplete list of references. 
Compared to them, we focus on a more practical setup of training two-layer non-linear networks with large stepsize GD. 
There are some general theories of EoS subject to subtle assumptions \citep[for example,][]{kong2020stochasticity,ahn2022understanding,ma2022beyond, damian2022self,wang2022analyzing,lu2023benign}, which are not directly comparable to ours.

In what follows, we make a detailed discussion about papers that directly motivate our work \citep{lyu2020gradient,ji2020directional,chatterji2021does,wu2023implicit,wu2024large}.

\paragraph{Comparison with \citet{lyu2020gradient,ji2020directional}.}
Both results in \citep{lyu2020gradient,ji2020directional} focused on $L$-homogenous networks. 
Specifically, \citet{lyu2020gradient} showed that a modified version of normalized margin (see \Cref{eq: norm margin}) induced by GD with small stepsize (such that the risk decreases monotonically) increases, with limiting points of $(\wB_t/\|\wB_t\|)_{t=1}^{\infty}$ converging to KKT points of a margin-maximization problem. 
Under additional o-minimal conditions, \citet{ji2020directional} showed that gradient flow converges in direction. 
Our work is different from theirs in two aspects. 
First, we allow GD with a large stepsize that may cause risk oscillation. 
Second, our theory covers non-homogenous predictors, which include two-layer networks with many commonly used activation functions beyond the scope of \citep{lyu2020gradient,ji2020directional}.
Compared to \citet{lyu2020gradient,ji2020directional}, we only show the improvement of the margin, and our theory is limited to nearly $1$-homogenous predictors (\Cref{assump: activation:near-homogeneous}).
It remains open to show directional convergence and to extend our near $1$-homogenity condition to a ``near $L$-homogeneity'' condition.


\paragraph{Comparison with \citet{chatterji2021does}.} 
The work by \citet{chatterji2021does} studied the convergence of GD in training deep networks under logistic loss. 
Their results are related to ours as we both consider networks with nearly homogeneous activations and we both have a stable phase analysis (although this is not explicitly mentioned in their paper). 
However, our results are significantly different from theirs. 
Specifically, in our notation, they require the homogenous error $\kappa$ (see \Cref{assump: activation:near-homogeneous}) to be smaller than $\Ocal(\log(1/L(\wB_s)) / \|\wB_s\|)\approx \Ocal(\bar \gamma(\wB_s))$, where $s$ is the time for GD to enter the stable phase. 
Note that the margin when GD enters the stable phase could be arbitrarily small. 
In comparison, we only require the homogenous error to be bounded by a constant. 
As a consequence, we can handle many commonly used activation functions (see \Cref{eg: homo act}) while they can only handle the Huberized ReLU with a small $h$ in \Cref{eg: homo act}.
Moreover, they require the stepsize $\tilde\eta$ to be smaller than $\Ocal(\kappa / \|\wB_s\|^8)$, thus they only allow very small stepsize. In contrast, we allow $\tilde \eta$ to be an arbitrarily large constant.

\paragraph{Comparison with \citet{wu2023implicit,wu2024large}.}
The works by \citet{wu2023implicit,wu2024large} directly motivate our paper. 
In particular, for logistic regression on linearly separable data, \citet{wu2023implicit} showed margin maximization of GD with large stepsize and \citet{wu2024large} showed fast optimization of GD with large stepsize. 
Our work can be viewed as an extension of \citep{wu2023implicit,wu2024large} from linear predictors to non-linear predictors such as two-layer networks. 
Besides, our results for margin improvement and convergence within the stable phase (\Cref{thm: Implicit bias and bound of stable phase}) hold for the general dataset, while their results strongly rely on the linear separability of the dataset.



\section{Conclusion}\label{sec:conclusion}
We provide a theory of large stepsize gradient descent (GD) for training non-homogeneous predictors such as two-layer networks using the logistic loss function. 
Our analysis explains the empirical observations: large stepsize GD often reveals two distinct phases in the training process, where the empirical risk oscillates in the beginning but decreases monotonically subsequently. 
We show that the phase transition happens because the average empirical risk decreases despite the risk oscillation.
In addition, we show that large stepsize GD improves the normalized margin in the long run, which extends the existing implicit bias theory for homogenous predictors to non-homogenous predictors. 
Finally, we show that large stepsize GD, by entering the initial oscillatory phase, achieves acceleration when minimizing the empirical risk. 

\section*{Acknowledgements}
We thank Fabian Pedregosa for his suggestions on an early draft and Jason D. Lee and Kaifeng Lyu for their comments clarifying the applicability of our result to sigmoid networks. 
We gratefully acknowledge the support of the NSF for FODSI through grant DMS-2023505, of the NSF and the Simons Foundation for the Collaboration on the Theoretical Foundations of Deep Learning through awards DMS-2031883 and \#814639, and of the ONR through MURI award N000142112431.

\bibliography{ref}
\appendix


\section{Stable Phase Analysis}\label{sec:stable-phase}

In this section, we will prove results for a general smooth predictor $f(\wB;\xB)$ under the logistic loss in the stable phase. Before the proof, we introduce some notations here. 

\paragraph{Notation.} We use the following notation to simplify the presentation. 
\begin{itemize}
    \item $    q_i(t) \coloneqq y_i f(\wB_t; \xB_i),\; q_{\min}(t) \coloneqq \min_{i\in [n]} q_i(t).$
    \item $L_t \coloneqq L(\wB_t), \;\rho_t := \|\wB_t\|_2.$
\end{itemize}
Then, we have the following expression: 
\[
    L(\wB_t) = \frac{1}{n}\sum_{i=1}^n \ell(q_i(t)). 
\]

Here, we give a summary of this section. The proofs are organized into 5 parts. 
\begin{itemize}
    \item In \Cref{sec:decrease-of-loss}, we characterize the decrease of loss $L_t$. 
    \item In \Cref{sec:increase-of-param}, we characterize the change of the parameter norm $\rho_t$. 
    \item In \Cref{sec:margin}, we show the convergence of the normalized margin $\bar \gamma(\wB_t)$. 
    \item In \Cref{sec:sharp-rate}, we characterize the sharp rates of loss $L_t$ and parameter norm $\rho_t$.
     \item In \Cref{sec:thm-stable-phase}, we give the proof of \Cref{thm: Implicit bias and bound of stable phase}. 
\end{itemize}

\subsection{Decrease of the Loss} \label{sec:decrease-of-loss}

In this section, we will show that the loss $L_t$ decreases monotonically in the stable phase. To begin with, we introduce the following definition which is another characterization of $\beta$-smoothness. 
\begin{definition}
[Linearization error]
\label{def: Linearized error}
Given a continuously differentiable function $f: \Rbb^d \to \Rbb$ and two points $\wB,\vB \in \Rbb^d$, the linearization error of $f(\vB)$ with respect to $\wB$ is: 
\[
    \xi[f](\wB,\vB) := f(\vB) - f(\wB)  - \nabla f(\wB)^\top (\vB-\wB).
\]
\end{definition}
For a $\beta$-smooth function, standard convex optimization theory gives the following linearization error bound.
\begin{fact}
    [Linearization error of $\beta$-smooth function]
    \label{lem: Linearized error of beta-smooth function} 
    For a $\beta$-smooth function $f:\Rbb^d \to \Rbb$, we have
  \[ \xi[f](\wB,\vB) :=  f(\vB) - f(\wB) -  \nabla f(\wB)^\top  (\vB-\wB) \le  \frac{\beta}{2}\|\vB -\wB\|_2^2,\quad \text{for every $\wB$ and $\vB$}. \]
 \end{fact}

We first show a stable phase bound for general smooth and Lipschitz predictors. The following is an extension of Lemma 10 in \citep{wu2024large}. Since we do not require $f$ to be twice differentiable, extra efforts are needed. 

\begin{lemma}
[Self-boundedness of logistic loss]
\label{lem: Self-bounded of logistic loss]}
For the logistic loss $\ell(z):=\log (1+\exp (-z))$, we have
$$
0\le \ell(z) -\ell(x)-\ell^{\prime}(x)(z-x)\le 2\ell(x)(z-x)^2
$$
for $|z-x|<1$.
\end{lemma}
\begin{proof}[Proof of \Cref{lem: Self-bounded of logistic loss]}]
See the proof of Proposition 5 in \citep{wu2024large}. The lower bound is by the convexity of $\ell(\cdot)$.
\end{proof}

The next lemma controls the decrease of the risk $L_t$.
\begin{lemma}
[Decrease of $L_t$]
\label{lem: Decrease of Lt}
Suppose \Cref{assump:model:bounded-grad,assump:model:smooth} hold.
If $L(\wB_t) \le \frac{1}{\tilde \eta \rho^2}$, then we have 
\[
    -\tilde\eta (1+ \beta \tilde \eta L(\wB_t)) \|\nabla L(\wB_t)\|^2  \le L(\wB_{t+1}) - L(\wB_t) \le -\tilde\eta (1-(2\rho^2 +  \beta )\tilde \eta L(\wB_t)) \|\nabla L(\wB_t)\|^2. 
\]
Particularly, this indicates that if $L(\wB_t) \le \frac{1}{\tilde \eta (2\rho^2 + \beta)}$, then $L(\wB_{t+1}) \le L(\wB_t)$. 
\end{lemma}
\begin{proof}[Proof of \Cref{lem: Decrease of Lt}]
By \Cref{assump:model:bounded-grad,assump:model:smooth}, we have $\| \nabla f\|_2 \le \rho$ and $f(\wB;\xB)$ is $\beta$-smooth as a function of $\wB$. Therefore, for every $i\in[n]$ we have
\[
    \begin{aligned}
        \left|q_i(t+1)-q_i(t)\right| & =\left|y_i(f\left(\wB_{t+1} ; \xB_i\right)-f\left(\wB_t ; \xB_i\right))\right| \\
        & =\left|\nabla f\left(\wB_t+\theta\left(\wB_{t+1}-\wB_t\right) ; \xB_i\right)^{\top}\left(\wB_{t+1}-\wB_t\right)\right| && \explain{by intermediate value theorem} \\
        & \leq \rho \left\|\wB_{t+1}-\wB_t\right\| \\ 
        & \le \rho\tilde \eta \|\nabla L_t\| && \explain{since $\wB_{t+1} = \wB_t - \tilde \eta \nabla L_t$} \\ 
        &\le \rho^2 \tilde \eta L_t  \le 1.  && \explain{since $\|\nabla L_t \| \le L_t \rho$}
        \end{aligned}
\]
Then by \Cref{lem: Self-bounded of logistic loss]}, we have
\begin{align*}
    \ell(q_i(t+1)) &\le \ell(q_i(t)) + \ell^\prime(q_i(t)) (q_i(t+1) - q_i(t)) + 2\ell(q_i(t)) (q_i(t+1) - q_i(t))^2\\ 
    &\le \ell(q_i(t)) + \ell^\prime(q_i(t)) \langle y_i \nabla f(\wB_t; \xB_i), \wB_{t+1} - \wB_t \rangle  + |\ell'(q_i(t))| \cdot | \xi[f](\wB_t, \wB_{t+1})| \\ 
    &\quad + 2\ell(q_i(t)) (q_i(t+1) - q_i(t))^2 \\  
    & \quad \quad \explain{ since $q_{i}(t+1) - q_i(t) = \langle y_i \nabla f(\wB_t; \xB_i), \wB_{t+1} - \wB_t  \rangle  + y_i \xi[f] (\wB_t, \wB_{t+1})$ } \\
    &\le \ell(q_i(t)) + \ell^\prime(q_i(t)) \langle y_i \nabla f(\wB_t; \xB_i), \wB_{t+1} - \wB_t \rangle  + \ell(q_i(t)) (\beta + 2\rho^2) \|\wB_{t+1} - \wB_t\|^2 .\\ 
    &\quad \quad \explain{ by \Cref{lem: Linearized error of beta-smooth function} and the previous inequality }
\end{align*}
Taking an average over all data points, we have 
\[
    L_{t+1} \le L_t  -\tilde \eta  \|\nabla L_t\|^2 + (2\rho^2 + \beta) \tilde \eta^2 L_t \|\nabla L_t\|^2,
\]
which is equivalent to
\[
    L_{t+1} - L_t \le -\tilde \eta (1-(2\rho^2 + \beta) \tilde \eta L_t) \|\nabla L_t\|^2.
\]
We complete the proof of the right hand side inequality. The left hand side inequality can be proved similarly. In detail, we can show that: 
\begin{align*}
            \ell(q_i(t+1)) 
            &\ge \ell(q_i(t)) + \ell^{\prime}(q_i(t)) (q_i(t+1) - q_i(t)\\
            &\ge \ell(q_i(t)) + \ell^\prime (q_i(t)) \langle y_i \nabla f(\wB_t;\xB_i), \wB_{t+1} - \wB_t \rangle - |\ell^\prime (q_i(t))| \cdot | \xi[f] (\wB_t, \wB_{t+1})|.  
\end{align*}
Taking the average over all data points, we have 
\[
    L_{t+1} \ge L_t -\tilde \eta (1+ \beta \tilde \eta L_t) \|\nabla L_t\|^2.  
\]
Now we have completed the proof of \Cref{lem: Decrease of Lt}. 
\end{proof}

\subsection{Increase of the Parameter Norm}\label{sec:increase-of-param}

In this section, we demonstrate that the parameter norm, $\rho_t$, increases monotonically during the stable phase. We introduce a crucial quantity, $v_t$, defined as the inner product of the gradient and the negative weight vector:
\[
    v_t \coloneqq \langle \nabla L(\wB_t), -\wB_t \rangle.
\]
This quantity, $v_t$, plays a key role in controlling the increase of the parameter norm. Notably, $v_t$ appears as the cross term in the expression $\|\wB_{t+1}\|^2 = \|\wB_t - \tilde \eta \nabla L(\wB_t)\|^2$. By managing $v_t$, we can effectively characterize the increase in the parameter norm.

Recall that our loss function is $\ell(x) := \log(1+e^{-x})$. Inspired by \citet{lyu2020gradient}, we define the following two auxiliary functions for the logistic loss:
\begin{equation}
\label{eq: psi, iota}
\begin{aligned}
    \psi(x) &\coloneqq -\log (\ell(x)) = -\log\log(1+e^{-x}), \quad x\in\Rbb, \\ 
    \iota(x) &\coloneqq \psi^{-1}(x) = -\log (e^{e^{-x}}-1),\quad x\in\Rbb. 
\end{aligned}
\end{equation}
One important remark is that if we change the loss to the exponential loss, both $\psi$ and $\iota$ will be the identity function. Since the logistic loss and the exponential loss have similar tails, our $\psi(x)$ and $\iota(x)$ are close to the identity function, i.e., 
\[
    \psi(x) \approx \iota(x) \approx x, \quad \text{for $x$ large enough.}
\]
Then, we have an exponential-loss-like decomposition of $L_t$:
\begin{equation}
\label{eq: decomp loss}
      L_t = \frac{1}{n} \sum_{i=1}^n \ell(q_i(t)) 
      = \frac{1}{n} \sum_{i=1}^n e^{-\psi(q_i(t))}.
\end{equation}
These two functions $\psi, \iota$ will help us to analyze the lower bound of $v_t$. First, we list some properties of $\psi$ and $\iota$ here. 
\begin{lemma}
[Auxiliary functions of $\ell$]
\label{lem: Auxiliary functions of ell}
The following claims hold for $\ell$, $\psi$, and $\iota$.
\begin{itemize}[leftmargin=*]
    \item $\ell(x) = e^{-\psi(x)}$. 
    \item $\ell$ is monotonically decreasing, while $\psi$ and $\iota$ are monotonically increasing. 
    \item  $\psi^\prime ( \iota(x)) = \frac{1}{\iota^\prime (x)}$; 
    \item $\psi^\prime (x) x $ is increasing for $x\in (0, +\infty)$.
\end{itemize}
\end{lemma}
\begin{proof}[Proof of \Cref{lem: Auxiliary functions of ell}]
The first two properties are straightforward. For the third property, we apply chain rule on $\psi(\iota(x)) = x$ to get
\[
    \psi^\prime (\iota(x)) \iota^\prime (x) = 1.
\]
For the fourth property, notice that
\[
    \psi^\prime (x) x = \frac{x}{(1+e^x)\log(1+e^{-x})}. 
\]
The denominator is positive and decreasing since 
\[
    \frac{d}{dx} \big[(1+e^x)\log(1+e^{-x})\big] = e^x \log(1+e^{-x}) - 1 \le e^x e^{-x} -1 =0. 
\]
Combining this with the fact that $x$ is positive and increasing, we have completed the proof of \Cref{lem: Auxiliary functions of ell}. 
\end{proof}

Besides, we have the following property of $\iota$. This is the key lemma to handle the homogeneous error. Actually, this lemma is another way to show $\iota(x)$ is close to the identity function.
    
    \begin{lemma}
    [Property of $\iota$]
    \label{lem: Property of iota]}
    For every $x\in\Rbb$, we have 
    \[
    \frac{\iota (x)}{\iota^\prime (x)} \ge x+\log \log 2. 
    \]
    \end{lemma}
    \begin{proof}[Proof of \Cref{lem: Property of iota]}]
Recall that
\[
    \iota(x) = -\log (e^{e^{-x}}-1) , \quad \iota^\prime (x) = \frac{e^{e^{-x}}e^{-x}}{e^{e^{-x}}-1}.
\]
    Let $y = e^{-x}$. We have 
    \[
    \frac{\iota (x)}{\iota^\prime (x)} = \frac{ -\log (e^{e^{-x}}-1) (e^{e^{-x}}-1)}{e^{e^{-x}}e^{-x}}= \frac{-\log (e^{y}-1)(e^{y}-1)}{e^y y}. 
    \]
Define $s(y) \coloneqq \frac{\iota(x)}{\iota^\prime (x)} - x - \log\log 2$. Then, 
\begin{align*}
        s(y) 
        &= \frac{- \log(e^{y}-1)(e^y-1)}{e^y y} +\log(y) - \log \log 2, \\ 
            s^\prime (y) &= \log(e^y-1) \cdot \underbrace{\frac{e^y-y-1}{e^y y^2}}_{>0}. 
\end{align*}

    Note that the sign of $s^\prime $ is determined by $\log(e^y-1)$. For $0<e^y\le 2$, $s^\prime (y)\le 0$ and $s(y)$ is decreasing; for $e^y\ge 2$, $s(y)$ is increasing. Therefore, 
\[
    \min_{y\in(0,\infty)} s(y) = s(\log 2) = 0. 
\]
Since $x = -\log y$, we have completed the proof of \Cref{lem: Property of iota]}.
\end{proof}

Another important property of $\iota$ is that it can provide a lower bound for $q_{\min}(t)$.
\begin{lemma}
[$\iota$ bound $q_{\min}$] 
\label{lem: iota bound qmin}
    For every $t\ge 0$, we have 
    \[
        q_{\min}(t) \ge \iota \big ( -\log(L_t) - \log n \big).
    \]
\end{lemma}
\begin{proof}[Proof of \Cref{lem: iota bound qmin}]
We use \cref{eq: decomp loss} to get
\begin{align*}
    \frac{1}{n} \ell(q_{\min}(t)) \le L_t &\Rightarrow \frac{1}{n}e^{-\psi( q_{\min}(t))} \le L_t\\ 
    &\Rightarrow \psi(q_{\min}(t)) \ge -\log n - \log L_t\\ 
    &\Rightarrow q_{\min}(t) \ge \iota \big ( -\log(L_t) - \log n \big).  && \explain{by \Cref{lem: Auxiliary functions of ell}}
\end{align*}
Then, we complete the proof of \Cref{lem: iota bound qmin}. 
\end{proof}

Now, we are ready to give a lower bound of $v_t$. The following lemma is an extension of Corollary E.6 in \citet{lyu2020gradient}, where they dealt with a homogeneous model and the exponential loss; we extend this to a non-homogeneous model. The key ingredient is \Cref{lem: Property of iota]}. 
\begin{lemma}
    [A lower bound of $v_t$]
    \label{lem: lowerbound of v_t}
    Suppose \Cref{assump:model:near-homogeneous} holds.
    Consider $v_t:=\la \grad L(\wB_t), -\wB_t\ra$.
    If $L_t \le \frac{1}{2n e^{\kappa}} $, then
    \[
        v_t \ge - L_t \log (2 n e^{\kappa} L_t)\ge 0. 
    \]
    \end{lemma}
    
    \begin{proof}[Proof of \Cref{lem: lowerbound of v_t}]
    By definition, we have
    \begin{align*}
        v_t &:= \langle \nabla L(\wB_t), -\wB_t \rangle \\ 
        & = -\frac{1}{n} \sum_{i=1}^n \ell'(y_i f(\wB_t; \xB_i)) y_i \langle \nabla f(\wB_t; \xB_i), \wB_t \rangle \\
        &=-\frac{1}{n} \sum_{i=1}^n \ell'(y_i f(\wB_t; \xB_i)) y_if(\wB_t; \xB_i)  -\frac{1}{n} \sum_{i=1}^n \ell'(y_i f(\wB_t; \xB_i)) \Big(y_i \langle \nabla f(\wB_t; \xB_i), \wB_t \rangle - y_if(\wB_t;\xB_i) \Big)\\ 
        &\ge -\frac{1}{n} \sum_{i=1}^n \ell'(y_i f(\wB_t; \xB_i)) y_i f(\wB_t; \xB_i) - \kappa L_t \\  
        &\quad \explain{since $|\ell^\prime (x) | \le \ell(x)$ and $|\langle\nabla f(\wB_t;\xB_i) ,\wB_t\rangle-f(\wB_t;\xB_i)| \le \kappa$ by \Cref{assump:model:near-homogeneous}}\\  
        & = \frac{1}{n} \sum_{i=1}^n e^{-\psi(q_i(t))} \psi^\prime (q_i(t)) q_i(t) - \kappa L_t. \qquad \explain{since $\ell(\cdot) = \exp(-\psi(\cdot))$}
    \end{align*}
    Applying \Cref{lem: iota bound qmin} and  \Cref{lem: Auxiliary functions of ell}, we have 
    \[
        q_i(t) \ge q_{\min}(t) \ge  \iota \big ( -\log(nL_t) \big) := -\log(e^{nL_t}-1) \ge -\log(e^{\frac{1}{2}}-1) \ge 0 . 
    \]
    Then we can apply \Cref{lem: Auxiliary functions of ell} to get
    \[
        \psi^\prime (q_i(t)) q_i(t) \ge  \psi^\prime \Big(\iota \big ( -\log(nL_t) \big)\Big) \iota \big ( -\log(nL_t) \big) = \frac{\iota \big ( -\log(nL_t) \big)}{\iota^\prime  \big ( -\log(nL_t) \big)}. 
    \]
    Invoking \Cref{lem: Property of iota]}, we have  
    \[
        \frac{\iota \big ( -\log(nL_t) \big)}{\iota^\prime  \big ( -\log(nL_t) \big)} \ge -\log(nL_t) + \log \log 2 \ge-\log(nL_t) + \log \log e^{\frac{1}{2}} = -  \log (2nL_t) . 
    \]
    Putting the above two inequalities together, we have 
    \[
        \psi^\prime (q_i(t)) q_i(t) \ge -\log(2nL_t),\quad \text{for every $i=1,\dots,n$}. 
    \]
    Plugging this back to the bound of $v_t$, we get 
    \begin{align*}
        v_t &\ge -\frac{1}{n} \sum_{i=1}^n e^{-\psi(q_i(t))}\log(2nL_t)  - \kappa L_t\\ 
        & = -L_t \log(2n L_t) - \kappa L_t\\ 
        & = - L_t \log(2n e^\kappa L_t)  \ge 0.
    \end{align*}
    This completes the proof of \Cref{lem: lowerbound of v_t}. 
    \end{proof}
    Right now, we get a lower bound for $v_t$, which is the cross term in the expression of $\|\wB_{t+1}\|^2$. The next lemma controls the increase of the parameter norm $\rho_t$ using $v_t$ and $L_t$. 
    \begin{lemma}
    [The increase of $\rho_t$]
    \label{lem: increase of rho in general model}
    Suppose \Cref{assump:model:bounded-grad,assump:model:near-homogeneous} hold.
    If $L_t\le \min \Big \{\frac{1}{2n e^{\kappa}}, \frac{1}{\tilde \eta (4\rho^2 + 2 \beta)}\Big\}$, then
    \[
       0\le  2\tilde \eta v_t\le \rho_{t+1}^2 - \rho_t^2 \le 2\tilde \eta  v_t \cdot \bigg( 1- \frac{1}{2\log(2ne^{\kappa}L_t)}\bigg).
    \]
    \end{lemma}
    \begin{proof}[Proof of \Cref{lem: increase of rho in general model}]
    By definition, we have
    \begin{align*}
        \rho_{t+1}^2 - \rho_t^2 & = 2\tilde \eta \langle \nabla L_t, -\wB_t \rangle + \tilde \eta^2 \|\nabla L_t\|^2\\ 
        &=2\tilde \eta v_t + \tilde \eta^2 \|\nabla L_t\|^2\ge 2\tilde \eta v_t \ge0 ,
    \end{align*}
    where the last inequality is by \Cref{lem: lowerbound of v_t}. Besides,
    \begin{align*}
        \rho_{t+1}^2 - \rho_t^2 
        & = 2\tilde \eta v_t \bigg( 1+ \frac{\tilde \eta \|\nabla L_t\|^2}{2v_t} \bigg) \\ 
        &\le 2\tilde \eta v_t \bigg( 1+ \frac{\tilde \eta L_t^2 \rho^2}{2v_t} \bigg)  &&\explain{by $\ell'\le \ell$, \Cref{assump:model:bounded-grad}, and \Cref{lem: lowerbound of v_t}}\\
        &\le2\tilde \eta v_t \bigg(1+ \frac{L_t}{2v_t} \bigg) &&\explain{by $L_t\le \frac{1}{\tilde \eta (4\rho^2 + 2 \beta)}$ } \\ 
        & \le 2\tilde \eta v_t \bigg( 1- \frac{1}{2\log(2ne^{\kappa}L_t)}\bigg). &&\explain{by \Cref{lem: lowerbound of v_t} } 
    \end{align*} 
    This completes the proof of \Cref{lem: increase of rho in general model}.
    \end{proof}
\subsection{Convergence of the Margin} \label{sec:margin}

In this section, we show that the normalized margin of a general predictor converges in the stable phase.
Recall that we define the (normalized) margin as
\[
    \bar \gamma(\wB) := \frac{\min_{i\in [n]} y_i f(\wB;\xB_i)}{\|\wB\|_2}. 
\]
However, this normalized margin is not a smooth function of $\wB$. Instead, we consider a smoothed margin $\gamma^a$ as an easy-to-analyze approximator of the normalized margin \citep{lyu2020gradient}
\begin{equation}
\label{eq: smooth margin}
    \gamma^a(\wB) := \frac{-\log L(\wB_t)}{\|\wB\|_2}. 
\end{equation}

We see that $\gamma^a$ is a good approximator of $\bar \gamma$. 
We can then use $\gamma^a$ to analyze the convergence of the normalized margin since they share the same limit (if it exists).
While $\gamma^a$ is relatively easy to analyze for gradient flow \citep{lyu2020gradient}, analyzing that for GD with a large (but fixed) stepsize is hard.
To mitigate this issue, we construct another two margins that work well with large stepsize GD following the ideas of \citet{lyu2020gradient}. 

Under \Cref{assump:model}, we define an auxiliary margin as
\begin{equation}
    \label{eq: auxiliary margin}
    \gamma^b(\wB) \coloneqq \frac{-\log(2ne^{\kappa}L(\wB))}{\|\wB\|},
\end{equation}
and a modified margin as 
\begin{equation}
    \label{eq: modified margin}
    \gamma^c (\wB) := \frac{e^{\Phi(L(\wB))}}{\|\wB\|}, 
 \quad    \text{where}\ 
    \Phi(x) := \log ( -\log(2ne^{\kappa}x))  + \frac{1+(4\rho^2 + 2\beta) \tilde \eta}{\log(2ne^{\kappa}x)}.
\end{equation}
These two margins provide a second-order correction when viewing large stepsize GD as a first-order approximation of gradient flow.  In the following discussion, we will show that $\bar \gamma(\wB)\approx \gamma^a(\wB)\approx \gamma^b(\wB)\approx \gamma^c(\wB)$. At last, we will use the convergence of $\gamma^c(\wB_t)$ to prove $\bar \gamma(\wB_t)$ converges.

The following lemma shows that $\bar \gamma(\wB)\approx \gamma^a(\wB)$. 
\begin{lemma}[Smoothed margin]
\label{lem: smooth margin as a good approximiator}
For the smooth margin $\gamma^a(\wB_t)$ defined in \cref{eq: smooth margin} and the normalized margin $\bar \gamma(\wB_t)$ defined in \cref{eq: norm margin}, we have 
\begin{itemize}[leftmargin=*]
    \item When $L_t \le \frac{1}{2n}$, we have 
\[
    q_{\min}(t) \le -\log L_t \le \log(2n) + q_{\min}(t), 
\]
    and 
    \[
        \bar \gamma(\wB_t) \le \gamma^a(\wB_t) \le \bar \gamma (\wB_t) + \frac{\log(2n)}{\rho_t}. 
    \]
    \item Assume \Cref{assump:model:near-homogeneous} holds. If $L_t \to 0$, then  $|\gamma^a(\wB_t) - \bar \gamma(\wB_t)| \to 0$.
\end{itemize}
\end{lemma}
\begin{proof}[Proof of \Cref{lem: smooth margin as a good approximiator}]

To prove the first claim, notice that  
\[L_t\le \frac{1}{2n} \implies \ell(q_{\min}(t)) = \log (1+ \exp(-q_{\min}(t)))\le nL_t\le \frac{1}{2}.\] 
Therefore we have  
\[e^{-q_{\min}(t)} \le e^{\frac{1}{2}} -1 \le 1.\] 
Using  $\frac{x}{2} \le \log(1+x) \le x$ for $0\le x\le 1$, we get 
\[
    \frac{1}{2} e^{-q_{\min}(t)} \le \ell(q_{\min}(t)) = \log(1+e^{-q_{\min}(t)}) \le e^{-q_{\min}(t)}. 
\]
Then we can bound $L_t$ by
\begin{align*}
    \frac{1}{2n}e^{-q_{\min}(t)} \le \frac{1}{n}  \ell(q_{\min}(t)) \le L_t \le \ell(q_{\min}(t))  \le e^{-q_{\min}(t)},
\end{align*}
which is equivalent to  
\[
    q_{\min}(t) \le -\log L_t \le \log(2n) + q_{\min}(t). 
\]
Dividing both sides by $\rho_t$ proves the second claim:
\[
    \bar \gamma(\wB_t) := \frac{q_{\min}(t)}{\rho_t} \le \gamma^a(\wB_t) := \frac{-\log L_t}{\rho_t} \le \bar \gamma (\wB_t) + \frac{\log(2n)}{\rho_t} = \frac{\log (2n) + q_{\min}(t)}{\rho_t}.
\]

For the last claim, we only need to show that $\rho_t \to \infty$. This is because if $L_t\to 0$, we have for any $i\in [n]$, $y_i f(\wB_t; x_i) \to \infty$. Using $y_i f (\wB_t; x_i) \le C_{r,\kappa}\|\wB_t\|+C_r$ from \Cref{lem: near homo}, we have $\rho_t = \|\wB_t\|_2 \to \infty$. 

Now we have completed the proof of \Cref{lem: smooth margin as a good approximiator}.
\end{proof}


The following lemma shows that $\gamma^c(\wB) \approx \bar \gamma(\wB)$.
\begin{lemma}
[Modified and auxiliary margins]
\label{lem: Modified margin as a good approximiator}
Suppose that \Cref{assump:model} holds.
Let $\gamma^c(\wB_t)$ be the modified margin as defined in \Cref{eq: modified margin} and $\gamma^b(\wB_t)$ be the auxiliary margin as defined in \Cref{eq: auxiliary margin}. If $L_t \le \frac{1}{2ne^{\kappa+2}}$, there exists a constant $c$ such that 
\[
\gamma^c (\wB_t) \le \gamma^b(\wB_t)\le \bar \gamma(\wB_t) \le \bigg(1+  \frac{c}{\log (1/L(\wB_t))} \bigg)\gamma^c(\wB_t). 
\]
\end{lemma}
\begin{proof}[Proof of \Cref{lem: Modified margin as a good approximiator}]~

\noindent
{\bf Step 1. Proof of the first inequalities. } Notice that
\begin{align*}
     e^{\Phi(L_t)} &=   - \log(2n e^{\kappa} L_t) \cdot \exp \bigg( \frac{1+(4\rho^2 + 2\beta) \tilde\eta }{\log(2ne^{\kappa} L_t)}\bigg) \qquad\explain{using \Cref{eq: modified margin}} \\ 
     &\le -\log(2n e^{\kappa} L_t)  \quad \explain{since $L_t\le \frac{1}{2ne^{\kappa}}$, $\exp \bigg( \frac{1+(4\rho^2 + 2\beta) \tilde\eta }{\log(2ne^{\kappa} L_t)} \bigg)\le 1$, and $\log(2ne^{\kappa}L_t) >0$ }\\ 
      &\le -\log(L_t) - \log(2n) \\ 
      &\le q_{\min}(t). \qquad \qquad  \qquad \qquad \qquad\qquad\explain{By argument 1 in \Cref{lem: smooth margin as a good approximiator}}
\end{align*}
By the definition of $\gamma^b$ and $\gamma^c$ as in \Cref{eq: modified margin,eq: auxiliary margin}, we have
\[
    \gamma^c (\wB_t) := \frac{e^{\Phi(L(\wB_t))}}{\|\wB_t\|} \le \frac{-\log(2n e^{\kappa} L_t)}{\|\wB_t\|} \eqqcolon \gamma^b(\wB_t) \le  \frac{q_{\min}(t)}{\|\wB_t\|} =: \bar \gamma(\wB_t).
\]
This completes the proof of the first two inequalities. 

\noindent
{\bf Step 2. Proof of the third inequality.} First, we have 
\begin{align*}
    \frac{\bar \gamma(\wB_t)}{\gamma^c(\wB_t)} &= \frac{\bar \gamma(\wB_t)}{\gamma^b(\wB_t)} \cdot \frac{\gamma^b(\wB_t)}{\gamma^c(\wB_t)}\\ 
    & =  \frac{q_{\min}(t)}{ - \log (2ne^{\kappa}L_t)} \cdot \exp \bigg( \frac{1+ (4\rho^2 + 2\beta)\tilde \eta}{- \log (2ne^{\kappa}L_t)} \bigg)  &&\explain{By the definitions of $\bar \gamma, \gamma^b, \gamma^c$}\\ 
    & \le  \frac{-\log(L_t)}{- \log (2ne^{\kappa}L_t)} \cdot \exp \bigg( \frac{1+ (4\rho^2 + 2\beta)\tilde \eta}{- \log (2ne^{\kappa}L_t)} \bigg) &&\explain{Since $q_{\min}(t) \le \log(-L_t)$ by \Cref{lem: smooth margin as a good approximiator}}\\ 
    &= \bigg(1+  \frac{\log(2ne^{\kappa})}{- \log (2ne^{\kappa}L_t)} \bigg) \cdot  \exp \bigg( \frac{1+ (4\rho^2 + 2\beta)\tilde \eta}{- \log (2ne^{\kappa}L_t)} \bigg).
\end{align*}
To simplify the notation, we let $c_1 \coloneqq 1+ (4\rho^2 + 2\beta)\tilde \eta$ and $c_2 = \log(2ne^{\kappa})$. Since $L_t\le \frac{1}{2ne^{\kappa+2}}\Rightarrow - \log (2ne^{\kappa}L_t) \ge 2>1$, we have
\[
    \frac{1+ (4\rho^2 + 2\beta)\tilde \eta}{- \log (2ne^{\kappa}L_t)}= \frac{c_1}{- \log (2ne^{\kappa}L_t)}\le c_1. 
\]
Besides, given $x<c$, we have $e^x\le 1+e^c x$. Therefore, 
\[
    \exp \bigg( \frac{1+ (4\rho^2 + 2\beta)\tilde \eta}{- \log (2ne^{\kappa}L_t)} \bigg) = \exp \bigg( \frac{c_1}{- \log (2ne^{\kappa}L_t)} \bigg) \le  1+ \frac{c_1\exp(c_1)}{- \log (2ne^{\kappa}L_t)}. 
\]
Plugging this into the bound for $\bar \gamma(\wB_t)/ \gamma^c(\wB_t)$, we get
\begin{align*}
    \frac{\bar \gamma(\wB_t)}{\gamma^c(\wB_t)} &= \bigg(1+  \frac{c_2}{- \log (2ne^{\kappa}L_t)} \bigg) \cdot  \exp \bigg( \frac{c_1}{- \log (2ne^{\kappa}L_t)} \bigg)\\ 
    &\le \bigg(1+  \frac{c_2}{- \log (2ne^{\kappa}L_t)} \bigg) \cdot  \bigg( 1+ \frac{\exp(c_1) c_1}{- \log (2ne^{\kappa}L_t)}\bigg)\\ 
    &\le 1+ \frac{c_2 + \exp(c_1)c_1+c_2c_1 \exp(c_1)}{- \log (2ne^{\kappa}L_t)} &&\explain{Since $-\log (2ne^{\kappa}L_t) \ge 1$}\\ 
    &= 1+ \frac{c_2 + \exp(c_1)c_1+c_2c_1 \exp(c_1)}{- \log L_t-c_2}. 
\end{align*}
Note that $- \log L_t-c_2 \ge 2>1$. Because $\frac{x}{x-c_2}$ is decreasing when $x\ge c_2+1$, we have
\[
\frac{-\log L_t}{- \log L_t-c_2} \le c_2+1 \implies \frac{1}{- \log L_t-c_2} \le \frac{c_2+1}{-\log L_t}. 
\]
Plug this inequality into the previous bound for $\bar \gamma(\wB_t)/\gamma^c(\wB_t)$, we get 
\[
\frac{\bar \gamma(\wB_t)}{\gamma^c(\wB_t)} \le 1+ \frac{(c_2 + \exp(c_1)c_1 + c_2c_1\exp(c_1))(c_2+1)}{-\log L_t}. 
\]
Let $c \coloneqq (c_2 + \exp(c_1)c_1 + c_2c_1\exp(c_1))(c_2+1)$. This completes the proof of the third inequality, and thus completes the proof of \Cref{lem: Modified margin as a good approximiator}. 
\end{proof}
The next lemma shows the convexity of $\Phi$ defined in \Cref{eq: modified margin}. The convexity will help us analyze the change of $\gamma^c(\wB_t)$ in the gradient descent dynamics. Specifically, we are going to use the property that 
\[
    \Phi(x) - \Phi(y) \ge \Phi^\prime (y)(x-y), \quad \text{for all $x,y$}.
\]

\begin{lemma}
[Convexity of $\Phi$]
\label{lem: Convexity of phi for general model}
The function $\Phi(x)$ defined in \Cref{eq: modified margin} is convex for $0<x<  \frac{1}{2ne^{2+\kappa}}$.
\end{lemma}
\begin{proof}[Proof of \Cref{lem: Convexity of phi for general model}]
Check that 
\[
 \Phi'(x)=
\frac{1 - \frac{1}{\log (2ne^{\kappa} x)}(1+(4\rho^2 +2 \beta)\tilde \eta)}{x \log (2ne^{\kappa} x)},
\]
and that
\[
    \Phi^{\prime \prime}(x) = \frac{(1+(4\rho^2 +2 \beta)\tilde\eta)\cdot(2+ \log(2n e^{\kappa}x))  - \log^2(2n e^{\kappa}x) - \log(2n e^{\kappa}x)}{x^2 \log^3(2n e^{\kappa}x)}.
\]
Note that when $x\le \frac{1}{2ne^{2+\kappa}}$, we have $ \log(2n e^{\kappa}x) \le -2$, which implies
\[
    2+ \log(2n e^{\kappa}x)\le 0,\ \log(2n e^{\kappa}x)<0, \ \text{and}\ -  \log^2(2n e^{\kappa}x) -  \log(2n e^{\kappa}x) <0. 
\]
Plugging these into the previous equality, we have $\Phi^{\prime \prime}(x) \ge 0$ when $0<x< \frac{1}{2ne^{2+\kappa}}$. Now we have completed the proof of \Cref{lem: Convexity of phi for general model}.
\end{proof}



Before we dive into the proof of the monotonic increasing $\gamma^c(\wB_t)$, we show that $\gamma^c$ is bounded first. The convergence of $\gamma^c$ is a direct consequence of the monotonic increasing and the boundedness of $\gamma^c$.
\begin{lemma}
[An upper bound on $\gamma^c$, $\gamma^b$, $\gamma^a$ and $\gamma^c$]
\label{lem: bound of tilde gamma and hat gamma in general model}
When $L_t \le \Big \{\frac{1}{2n e^{\kappa}}, \frac{1}{\tilde \eta (4\rho^2 + 2 \beta)}\Big\}$ for $t\ge s$, there exists $B_0$ such that
\[
    \gamma^c(\wB_t)\le \gamma^b(\wB_t)\le \gamma^a(\wB_t) \le \bar \gamma(\wB_t) +\frac{\log 2n}{\rho_s}\le B_0.
\]
\end{lemma}
\begin{proof}[Proof of \Cref{lem: bound of tilde gamma and hat gamma in general model}]
Apply \cref{lem: increase of rho in general model}, we have $\|\wB_t\|\ge \rho_{t} \ge \rho_s$. Then we can apply \cref{lem: near homo} and there exists a constant $C_{\rho_s,\kappa}$ such that for all $i$,
\[
    |y_i f(\wB_t;\xB_i)| \le C_{\rho_s,\kappa} \|\wB_t\|.
\]
Hence, 
\[
    \bar \gamma(\wB_t)  = \frac{\arg\min_{i\in[n]} y_i f(\wB_t;\xB_i)}{\|\wB_t\|}\le C_{\rho_s, \kappa}. 
\]
Besides, by \Cref{lem: smooth margin as a good approximiator}, we have 
\[
    \gamma^a (\wB_t) \le \bar \gamma(\wB_t) +\frac{\log 2n}{\rho_t}\le C_{\rho_s, \kappa} + \frac{\log 2n}{\rho_s}. 
\]
By \Cref{lem: Modified margin as a good approximiator}, we have 
\[
    \gamma^c(\wB_t) \le \gamma^b(\wB_t) \le \gamma^a(\wB_t) \le C_{\rho_s, \kappa} + \frac{\log 2n}{\rho_s}.
\]
Let $B_0 = C_{\rho_s, \kappa} + \frac{\log 2n}{\rho_s}$. Then, we complete the proof of \Cref{lem: bound of tilde gamma and hat gamma in general model}.
\end{proof}

The following lemma is a variant of Proposition 5, item 1, in  \citep{wu2024large} and Lemma 4.8, in \citep{lyu2020gradient}. Before the lemma, we need some auxiliary definitions. 
let us define  
    \[\thetaB_t \coloneqq \frac{\wB_t}{\|\wB_t\|},\quad  \nuB_t \coloneqq \thetaB_t \thetaB_t^\top (-\nabla L_t),\quad 
    \muB_t \coloneqq\big(\IB -\thetaB_t \thetaB_t^\top\big) (-\nabla L_t).\] 
Therefore, we have 
\[
    \|\nabla L_t\|^2 = \|\nuB_t\|^2 + \|\muB_t\|^2. 
\]
The key point of this decomposition is that we consider the gradient of the loss function as a sum of two orthogonal components. The first component $\nuB_t$ is the component in the direction of $\wB_t$, and the second component $\muB_t$ is the component orthogonal to $\wB_t$. We will show that the modified margin $\gamma^c(\wB_t)$ is monotonically increasing. And the increase of $\gamma^c(\wB_t)$ is lower bounded by a term that depends on $\|\muB_t\|^2$.
\begin{lemma}
    [Modified margin is monotonically increasing]
    \label{lem: Modified margin is monotonically increasing for general model]}
    Suppose \Cref{assump:model} holds.
    If there exists $s$ such that 
    \[ L_s \le \min \bigg \{\frac{1}{e^{\kappa+2}2n},\ \frac{1}{\tilde \eta (4\rho^2 + 2 \beta)}\bigg\},\] 
    then for $t\ge s$, we have 
    \begin{itemize}[leftmargin=*]
    \item $L_{t+1} \le L_t$.
        \item $v_t \ge -L_t \log(2n e^{\kappa} L_t) \ge 0$.
        \item $\rho_{t+1}^2 - \rho_t^2 \ge 2\tilde \eta  v_t$.
        \item $\log \gamma^c(\wB_{t+1}) - \log \gamma^c (\wB_{t}) \ge  \frac{\rho_t^2}{v_t^2} \|\muB_t\|^2 \log \frac{\rho_{t+1}}{\rho_t}$.
    \end{itemize}
As a consequence, $\gamma^c(\wB_t)$ admits a finite limit. 
    \end{lemma}
    \begin{proof}[Proof of \Cref{lem: Modified margin is monotonically increasing for general model]}]
    The first claim is by \Cref{lem: Decrease of Lt} and induction.
    The second and the third claims are consequences of \Cref{lem: lowerbound of v_t,lem: increase of rho in general model}, respectively. 
    We now prove the last claim. The proof is divided into two steps. 
    \begin{itemize}
        \item In step 1, we constuct an auxilary function $\Psi(x)$ which is close to $-\Phi^{\prime}(x)$ and show that: 
            \[
            \Psi(L_t)(L_{t+1} - L_t) \le - \frac{1}{\rho_t^2}(\rho_t^2 - \rho_{t+1}^2) \big( \frac{1}{2} + \frac{\rho_t^2}{2v_t^2}\|\muB_t\|^2\big ). 
            \]
        \item In Step 2, we show that $\Psi(x)\le -\Phi^{\prime}(x)$ and use this property to get the monotonicity of the modified margin $\gamma^c$. 
    \end{itemize}

    \noindent 
    {\bf Step 1. Construct $\Psi(x)$.}
    
    \noindent
    By \Cref{lem: Decrease of Lt}, we have 
    \[
        L_{t+1} - L_t:= L(\wB_{t+1}) - L(\wB_t) \le -\tilde\eta (1-(2\rho^2 +  \beta )\tilde \eta L_t) \|\nabla L_t\|^2. 
    \]
    Multiplying both sides by $ 2 v_t  \frac{1- \frac{1}{2\log (2ne^{\kappa} L_t)}}{1-(2\rho^2 +  \beta)\tilde \eta L_t}>0$, we get 
    \[
        \frac{1- \frac{1}{2\log (2ne^{\kappa} L_t)}}{1-(2\rho^2 +  \beta)\tilde \eta L_t} 2v_t (L_{t+1} - L_t) \le -2\tilde \eta v_t  \bigg( 1- \frac{1}{2\log (2ne^{\kappa} L_t)}\bigg)  \|\nabla L_t\|^2. 
    \]
    From \Cref{lem: increase of rho in general model} we have \[ 0\le \rho_{t+1}^2 - \rho_t^2 \le 2\tilde \eta  v_t  \Big( 1- \frac{1}{2\log (2ne^{\kappa} L_t)}\Big).\] 
    Using the above we get
    \begin{equation}
    \label{eq: 523.1}
    \frac{1- \frac{1}{2\log (2ne^{\kappa} L_t)}}{1-(2\rho^2 +  \beta)\tilde \eta L_t} 2v_t (L_{t+1} - L_t) \le -(\rho_{t+1}^2 - \rho_t^2) \|\nabla L_t\|^2.
    \end{equation}

    Recall that
    \[
    \|\nabla L_t\|^2 = \|\nuB_t\|^2 + \|\muB_t\|^2. 
\]
    For $\nuB_t$, we have  
    \[
        \|\nuB_t\| = \frac{1}{\rho_t}\langle \wB_t, -\nabla L_t \rangle =  \frac{v_t}{\rho_t}. 
    \]
    Then we can decompose $\|\nabla L_t\|^2$ as 
    \begin{equation}\label{eq:grad-decomp}
        \|\nabla L_t \|^2 = \|\nuB_t\|^2 + \|\muB_t\|^2 = \frac{v_t^2}{\rho_t^2} + \|\muB_t\|^2.
    \end{equation}
    Plugging this into \Cref{eq: 523.1} and dividing both two sides by $2v_t^2$, we have 
    \[
        \frac{1- \frac{1}{2\log (2ne^{\kappa} L_t)}}{(1-(2\rho^2 +  \beta)\tilde \eta L_t)v_t}  (L_{t+1} - L_t) \le  -\frac{1}{\rho_t^2}(\rho_{t+1}^2 - \rho_{t}^2)\bigg( \frac{1}{2} + \frac{\rho_t^2}{2 v_t^2} \|\muB_t\|^2\bigg) .
    \]
    By \Cref{lem: lowerbound of v_t}, we have $v_t \ge - L_t \log(2ne^{\kappa} L_t)$. 
    Define 
    \[\Psi(x) := -\frac{1- \frac{1}{2\log (2ne^{\kappa} x)}}{(1-(2\rho^2 +  \beta)\tilde \eta x)x \log (2ne^{\kappa} x)}.\]
    Then we have  
    \begin{equation}
    \label{eq: core equa general}
    \begin{aligned}
    \Psi(L_t) (L_{t+1} - L_t) &\coloneqq-\frac{1- \frac{1}{2\log (2ne^{\kappa} L_t)}}{(1-(2\rho^2 +  \beta)\tilde \eta L_t) L_t \log (2ne^{\kappa} L_t)}  (L_{t+1} - L_t) \\ 
    &\le \frac{1- \frac{1}{2\log (2ne^{\kappa} L_t)}}{(1-(2\rho^2 +  \beta)\tilde \eta L_t)v_t}  (L_{t+1} - L_t)  \\
    & \le  -\frac{1}{\rho_t^2}(\rho_{t}^2 - \rho_{t+1}^2)\bigg( \frac{1}{2} + \frac{\rho_t^2}{2 v_t^2} \|\muB_t\|^2\bigg).
    \end{aligned}
    \end{equation}

    \noindent 
    {\bf Step 2. Show $\Psi(x)\le -\Phi^{\prime}(x)$.}

    \noindent
We are going to show that $\Psi(x) \le - \Phi^\prime (x)$. Note that when $0<x\le \min \Big \{\frac{1}{e^{\kappa+2}2n}, \frac{1}{\tilde \eta (4\rho^2 + 2 \beta)}\Big\}$, we have $\log (2n e^{\kappa}x) < 0$ and $1- (2\rho^2 + \beta) \tilde \eta x) \ge \frac{1}{2}>0$. Therefore, we have 
    \begin{equation}
    \label{eq: Psi and J}
               \Psi(x) = \underbrace{\frac{1- \frac{1}{2\log (2ne^{\kappa} x)}}{1-(2\rho^2 +  \beta)\tilde \eta x}}_{\eqqcolon J >0} \cdot \underbrace{\frac{-1}{x \log(2ne^{\kappa}x)}}_{>0}. 
    \end{equation}
    To get an upper bound of $\Psi(x)$, we just need an upper bound of $J$.  Let $a := \frac{-1}{2\log(2ne^{\kappa}x)}>0$ and $b := (2\rho^2 + \beta) \tilde \eta x\in (0,1/2]$. 
    Then we invoke \Cref{lem: bound for the phi and psi} to get 
    \[
        J := \frac{1+a}{1-b} \le 1+2a + 2b = 1 -  \frac{1}{\log(2ne^{\kappa}x)} + (4\rho^2 + 2\beta) \tilde \eta x. 
    \]
Recall that $x \le \frac{1}{e^{\kappa+2}2n}\le \frac{1}{2n e^{\kappa}}$ and $ 2n e^{\kappa} \ge 1$. 
Then we apply \Cref{lem: bound for x and 1/log(4nx)} to get $x\le \frac{-1}{\log (2n e^{\kappa}x)}$.
Plugging this into the bound of $J$, we get
    \begin{equation}
    \label{eq: bound of J} 
        J \le 1 -  \frac{1}{\log(2ne^{\kappa}x)} + (4\rho^2 + 2\beta) \tilde \eta x \le 1 -  \frac{1}{\log(2ne^{\kappa}x)}(1+ (4\rho^2 + 2\beta) \tilde \eta).
    \end{equation}
  Plugging \eqref{eq: bound of J} into \eqref{eq: Psi and J}, we have 
    \[
        \Psi(x) = J \cdot \frac{-1}{x \log(2ne^{\kappa}x)} \le - \frac{1 - \frac{1}{\log (2ne^{\kappa} x)}(1+(4\rho^2 +2 \beta)\tilde \eta)}{x \log (2ne^{\kappa} x)} = - \Phi'(x), 
    \]
   which verifies that $\Psi(x) \le - \Phi^\prime (x)$.
    By this and \Cref{eq: core equa general}, we have 
    \[
        \Phi'(L_t) (L_{t+1} - L_t) + \varphi^\prime(\rho_t^2) (\rho_{t+1}^2 - \rho_t^2) \bigg( \frac{1}{2} + \frac{\rho_t^2}{2 v_t^2} \|\muB_t\|^2\bigg) \ge 0, 
    \]
    where $\varphi(x) = -\log x = \log(1/x)$. 
    Recall that for $0<x\le \frac{1}{2ne^{\kappa+2}}$, $\Phi(x)$ is convex by \Cref{lem: Convexity of phi for general model}. By convexity of $\varphi$ and $\Phi$, we have 
    \[
        \Phi(L_{t+1}) - \Phi(L_t) + \bigg(\log \frac{1}{\rho_{t+1}^2} - \log \frac{1}{\rho_t^2} \bigg) \bigg( \frac{1}{2} + \frac{\rho_t^2}{2 v_t^2} \|\muB_t\|^2\bigg) \ge 0.  
    \]
    By the definition of $\gamma^c$ in \Cref{eq: modified margin}, this can be rewritten as 
    \begin{align*}
        \log \gamma^c(\wB_{t+1}) - \log \gamma^c (\wB_{t}) &= ( \Phi(L_{t+1}) - \Phi(L_t)) + \bigg(\log \frac{1}{\rho_{t+1}} - \log \frac{1}{\rho_t} \bigg) \\ 
        & \ge - \bigg(\log \frac{1}{\rho_{t+1}^2} - \log \frac{1}{\rho_t^2} \bigg)\frac{\rho_t^2}{2 v_t^2} \|\muB_t\|^2\\ 
        & = \frac{\rho_t^2}{v_t^2} \|\muB_t\|^2 \log \frac{\rho_{t+1}}{\rho_t} \\
        &\ge 0,
    \end{align*}
    where the last inequality is because of \Cref{lem: increase of rho in general model}.
    We have shown that $\gamma^c(\wB_t)$ is monotonically increasing. 
    By \Cref{lem: bound of tilde gamma and hat gamma in general model}, $\gamma^c(\wB_t)$ is bounded.
    Therefore $\gamma^c(\wB_t)$ admits a finite limit.
    This completes the proof of \Cref{lem: Modified margin is monotonically increasing for general model]}. 
    \end{proof}

\subsection{Sharp rates of Loss and Parameter Norm} \label{sec:sharp-rate}

Right now, we have already proved that $\gamma^c(\wB_t)$ is monotonically increasing and bounded, which indicates $\gamma^c(\wB_t)$ converges. However, if we want to show that $\bar \gamma(\wB_t)$ converges, we still need to verify that $L_t\to 0$, which is the crucial condition for $\gamma^c(\wB_t), \gamma^b(\wB_t), \gamma^a(\wB_t)$, and $\bar \gamma (\wB_t)$ to share the same limit, by \Cref{lem: smooth margin as a good approximiator} and \Cref{lem: Modified margin as a good approximiator}.

Fortunately, with the monotonicity of $\gamma^c(\wB_t)$, we can prove that $L_t$ converges to zero and even characterize the rate of $L_t$. 

\begin{lemma}
[Rate of $L_t$ in general model]
\label{lem: Order of Loss in general model}
Suppose \Cref{assump:model} holds.
If there is an $s$ such that  
\[L(\wB_s) \le \min \big\{\frac{1}{e^{\kappa+2}2n}, \frac{1}{\tilde \eta (4\rho^2 + 2 \beta)}\big\},\] 
then for every $t\ge s$ we have
\[
    \frac{1}{ \frac{1}{ L(\wB_s)} + 3\tilde\eta \rho^2(t-s)}\le L(\wB_{t}) \le\frac{2}{(t-s) \tilde \eta  \gamma^c(\wB_s)^2}. 
\]
That is, $L(\wB_t) = \Theta( 1/t) \to 0$ as $t\to\infty$.   
\end{lemma}
\begin{proof}[Proof of \Cref{lem: Order of Loss in general model}]
By \Cref{lem: Decrease of Lt} and \Cref{eq:grad-decomp} in the proof of \Cref{lem: Modified margin is monotonically increasing for general model]}, 
we know $L_t$ is decreasing and 
\begin{equation}
\label{eq: Lt vt rhot}
 L_{t+1} - L_t \le -\frac{\tilde \eta}{2} \|\nabla L_t\|^2 \le -\frac{\tilde \eta}{2} \|\nuB_t\|_2^2\le -\frac{\tilde \eta}{2} \frac{v_t^2}{\rho_t^2}. 
\end{equation}
We will establish an upper bound for $\rho_t$ first. Note that $\gamma^c(\wB_t)$ is increasing for $t\ge s$ by \Cref{lem: Modified margin is monotonically increasing for general model]} and $\gamma^b(\wB_t)\ge \gamma^c(\wB_t)$ by \Cref{lem: Modified margin as a good approximiator}. By \Cref{eq: auxiliary margin}, we have
\[
    \rho_t = \frac{-\log(2ne^{\kappa}L_t)}{\gamma^b(\wB_t)} \le \frac{-\log(2ne^{\kappa}L_t)}{\gamma^c(\wB_t)}\le \frac{-\log(2ne^{\kappa}L_t)}{\gamma^c(\wB_s)}. 
\]
Combining this with \Cref{lem: lowerbound of v_t}, we have 
\[
    \frac{v_t}{\rho_t} \ge \frac{-L_t\log(2ne^{\kappa}L_t)}{\frac{-\log(2ne^{\kappa}L_t)}{\gamma^c(\wB_s)}} = L_t \gamma^c(\wB_s). 
\]
Plugging this into \eqref{eq: Lt vt rhot}, we have 
\[
    L_{t+1} -L_t \le - \frac{\tilde \eta}{2} L_t^2 \gamma^c(\wB_s)^2,
\]
which implies 
\begin{align*}
        \frac{\tilde \eta\gamma(\wB_s)^2}{2} &\le \frac{L_t -L_{t+1}}{L_t^2}\\ 
            &\le \frac{L_t -L_{t+1}}{L_t L_{t+1}}   &&\explain{ Since $L_{t+1} \le L_t$} \\
            &= \frac{1}{L_{t+1} } - \frac{1}{L_{t}},\quad t\ge s.
 \end{align*}
Telescoping the sum from $s$ to $t$, we have 
\[
    (t-s) \frac{\tilde \eta\gamma^c(\wB_s)^2}{2} \le \frac{1}{L_{t}} - \frac{1}{L_s} \le \frac{1}{L_{t}}.  
\]
Therefore we have
\[
    L_t\le \frac{2}{(t-s) \tilde \eta  \gamma^c(\wB_s)^2}. 
\]

Next we show the lower bound on the risk.
By \Cref{lem: Decrease of Lt} we have
\[
    L_{t+1} - L_t \ge - \tilde \eta ( 1+  \beta \tilde \eta L_t) \|\nabla L_t\|^2 \ge \tilde -\frac{3}{2}\eta  \|\nabla L_t\|^2.
\]
Observe that under \Cref{assump:model:bounded-grad},
\[
    \|\nabla L_t\| = \bigg\|\frac{1}{n}\sum_{i=1}^n \ell'( q_i(t)) y_i \nabla f(\wB_t; \xB_i)\bigg\| \le \rho L_t. 
\]
Then we have
\[
    L_{t+1} - L_t \ge - \tilde \eta \frac{3}{2} \rho^2 L_t^2,\quad t\ge s.
\]
Let $\tilde L_t :=  \frac{3 \tilde \eta \rho^2}{2} L_t$, we have $\tilde L_s \le \frac{3 \tilde \eta \rho^2}{2} \frac{1}{\tilde \eta(4\rho^2 + 2\beta)}\le \frac{3}{8}\le \frac{1}{2}$. Furthermore, since $L_t$ decreases monotonically, $\tilde L_t\le \tilde L_s \le \frac{1}{2}$. The inequality becomes 
\[
    \tilde L_{t+1} - \tilde L_t \ge - \tilde L_t^2.
\] Therefore, let $c = \frac{1}{\tilde L_s}$ and apply \Cref{lem: bound for update lower}, we have  for any $t\ge s$,
\[
    \tilde L_{t} \ge \frac{1}{c+ 2 (t-s)}. 
\]
This is  equivalent to 
\[
    L_t \ge \frac{1}{ \frac{1}{ L_s} + 3\tilde\eta \rho^2(t-s)}.  
\]
We have completed the proof of \Cref{lem: Order of Loss in general model}. 
\end{proof}

Furthermore, we can characterize the order of $\rho_t$ in the stable phase.

\begin{lemma}
[Order of $\rho_t$ in general model]
\label{lem: Order of rho in general model}
Suppose \Cref{assump:model} holds.
If there is $s$ such that 
\[L(\wB_s) \le \min \bigg\{\frac{1}{e^{\kappa+2}2n},\ \frac{1}{\tilde \eta (4\rho^2 + 2 \beta)}\bigg\},\] 
then for $t\ge s$ we have  
\[
    \rho_t   = \Theta(\log(t)).
\]
\end{lemma}
\begin{proof}[Proof of \Cref{lem: Order of rho in general model}]
Note that $\gamma^c(\wB_t)$ is increasing for $t\ge s$ by \Cref{lem: Modified margin is monotonically increasing for general model]} and $\gamma^b(\wB_t)\ge \gamma^c(\wB_t)$ by \Cref{lem: Modified margin as a good approximiator}. Therefore, 
\[
    \rho_t \le \frac{-\log(2ne^{\kappa}L_t)}{\gamma^b(\wB_t)} \le \frac{-\log(2ne^{\kappa}L_t)}{\gamma^c(\wB_t)}\le \frac{-\log(2ne^{\kappa}L_t)}{\gamma^c(\wB_s)}. 
\]
Combining this with \Cref{lem: Order of Loss in general model}, we have 
\[
    \rho_t \le \frac{\log \frac{1/L(\wB_s) + 3\tilde \eta \rho^2 (t-s)}{2n e^{\kappa}}}{\gamma^c(\wB_s)} = \Ocal(\log(\tilde \eta t)).  
\]
Besides, we have $q_{\min}\ge \iota(\log \frac{1}{L}- \log n)$ by \Cref{lem: smooth margin as a good approximiator} and $q_{\min} \le B_0 \rho_t$ by \Cref{lem: bound of tilde gamma and hat gamma in general model}. Therefore we have 
\[
    \rho_t \ge \frac{\iota(\log\frac{1}{nL_t})}{B_0}\ge \frac{\log \frac{1}{nL_t}}{2B_0}\ge \frac{\log \frac{(t-s)\tilde \eta \gamma^c(\wB_s)^2}{2n}}{2B_0} = \Omega(\log( t)),
\]
where the second inequality is because for $\iota(x)$ defined in \eqref{eq: psi, iota}, $\iota(x) \ge \frac{x}{2}$ for $x \ge 0.6$, and the third inequality is by \Cref{lem: Order of Loss in general model}. Combining them, we get
\[
    \rho_t = \Theta(\log( t)). 
\]
This completes the proof of \Cref{lem: Order of rho in general model}. 
\end{proof}

\subsection{Proof of Theorem~\ref{thm: Implicit bias and bound of stable phase}} \label{sec:thm-stable-phase}
\begin{proof}[Proof of \Cref{thm: Implicit bias and bound of stable phase}]
We prove the items one by one. 
\begin{itemize}
    \item The monotonicity of $L_t$ comes from the result of \Cref{lem: Decrease of Lt} directly. 
    \item Item 1 is due to \Cref{lem: Order of Loss in general model} . 
    \item For item 2, the monotonicity of $\rho_t$ comes from the result of \Cref{lem: increase of rho in general model} and the order is due to \Cref{lem: Order of rho in general model}. 
    \item For item 3, we first know that $L_t\to 0$ by \Cref{lem: Order of Loss in general model}. Then, by \Cref{lem: Modified margin is monotonically increasing for general model]} and \Cref{lem: bound of tilde gamma and hat gamma in general model}, we know that $\gamma^c(\wB_t)$ converges. Combining these with \Cref{lem: smooth margin as a good approximiator} and \Cref{lem: Modified margin as a good approximiator}, we know that  $\gamma^c(\wB_t)$ is an $\big(1+O\big(1/(\log \frac{1}{L(\wB_t)}\big)\big)$-multiplicative approximation of$\bar \gamma(\wB_t)$, and $\bar \gamma(\wB_t)$ shares the same limit as $\gamma^c(\wB_t)$.
\end{itemize}
\end{proof}

\section{EoS Phase Analysis} 
In this section, we focus on the linearly separable case, that is, we work under \Cref{assump:separable}. We mainly follow the idea of \citep{wu2024large} for the proof. 
In detail, we consider a comparator
\[
\uB := \uB_1 + \uB_2,
\] 
where 
where 
\begin{equation}\label{eq:ub1 acc}
    \uB_1 := \begin{pmatrix}
        \uB_1^{(1)} \\
        \vdots \\
        \uB_1^{(m)}
    \end{pmatrix}, \ \
    \text{with}\ \  
    \uB_1^{(j)} := a_j \frac{\log(\gamma^2 \eta T) + \kappa }{\alpha \gamma } \cdot \wB_*, \quad j=1,\ldots ,m,
\end{equation}
and 
\begin{equation}
\label{eq:ub2 acc}
\uB_2 := \begin{pmatrix} 
        \uB_2^{(1)}\\ 
        \vdots\\ 
        \uB_2^{(m)}
        \end{pmatrix}, \ \ 
        \text{with}\ \  \uB_2^{(j)} := a_j \frac{\eta }{ 2\gamma } \cdot \wB_*, \quad j=1,\ldots ,m.
\end{equation}
Consider the following decomposition,
\begin{align*}
    \begin{aligned}
        \|\wB_{t+1}-\uB\|^2 & =\|\wB_t-\uB\|^2+2 m\eta\langle\nabla L (\wB_t ), \uB-\wB_t\rangle+m^2\eta^2\|\nabla L(\wB_t )\|^2 \\
        & =\|\wB_t-\uB\|^2+
        2 m\eta \underbrace{\left\langle\nabla L(\wB_t), \uB_1-\wB_t\right\rangle}_{=: I_1(\wB_t)}+ m
        \eta\big(\underbrace{2\langle\nabla L(\wB_t ), \uB_2\rangle+ m \eta\|\nabla L(\wB_t)\|^2}_{=: I_2(\wB_t)}\big)
        \end{aligned}. 
\end{align*}
We aim to prove $I_1(\wB_t) \le \frac{1}{T} - L(\wB_t)$ and $I_2(\wB_t) \le 0$. Then we can get a bound for the average loss by telescope summing the decomposition. Here we also introduced the following vector $\bar \wB_*$:
\[
\bar \wB_* := \begin{pmatrix}
    a_1\wB_* \\
    \vdots \\
    a_n\wB_*
\end{pmatrix}
\]
We can observe that $\uB_1 = \frac{\log(\gamma^2 \eta T)+ \kappa}{\alpha \gamma} \bar \wB_*$ and $\uB_2 = \frac{\eta}{2\gamma} \bar \wB_*$.
\begin{lemma}
[A bound on $I_1(\wB)$ in the EoS phase]
\label{lem: Bound of I1 in EoS}
For $\uB_1$ defined in \eqref{eq:ub1 acc}, we have 
\[
    I_1(\wB) := \langle\nabla L (\wB ), \uB_1-\wB \rangle \le \frac{1}{\gamma ^2 \eta T} - L(\wB).
\]
\end{lemma}
\begin{proof}[Proof of \Cref{lem: Bound of I1 in EoS}]
    Since $L$ is averaged over the individual losses incurred at the data $(\xB_i, y_i)_{i=1}^n$ and gradient is a linear operator, 
    it suffices to prove the claim assuming there is only a single data point
    $(\xB, y)$. 
    Then by \Cref{assump:separable}, we have
    \[
        \langle y\xB, \wB_* \rangle \ge \gamma >0.  
    \] 
    Then the loss becomes 
    \[
        L(\wB) = \ell (y  f(\wB; \xB)) = \ell \biggl (y  \frac{1}{m}\sum_{j=1}^m a_j \phi(\xB^\top  \wB^{(j)})\biggr ). 
    \]
    Now we expand $I_1(\wB)$:
    \begin{align}
    \label{eq: I1 acc}
        I_1(\wB)&:=\langle \nabla L(\wB), \uB_1 - \wB \rangle \notag\\ 
        &=\ell'\big(yf(\wB;\xB)\big)\langle y\nabla f(\wB;\xB), \uB_1 - \wB \rangle \notag\\ 
        &= \ell' \bigl (yf(\wB;\xB) \bigr )  \frac{1}{m}  \sum_{k=1}^m a_k y\phi'(x^\top \wB^{(k)}) \xB^\top ( \uB_1^{(k)} - \wB^{(k)})  \notag\\ 
        & = \ell' \bigl ( y f(\wB; \xB)   \bigr ) \Biggl[  \underbrace{ \frac{1}{m} \sum_{k=1}^m a_k y \Bigl (\phi'(\xB^\top \wB^{(k)}) \xB^\top \uB_1^{(k)} + \phi(\xB^\top \wB^{(k)}) - \phi'(\xB^\top \wB^{(k)}) \xB^\top  \wB^{(k)}\Bigr )}_{=:J_1} \notag \\
        &\hspace{80mm} - \underbrace{ \frac{1}{m} \sum_{k=1}^m a_k y\phi(\xB^\top \wB^{(k)}) }_{=: J_2} \Biggr]. 
    \end{align}
    By definition we have  $J_2 = y f(\wB; \xB). $ As for $J_1$, using $\phi'\ge \alpha $ and $a_k y x^\top  \uB_1^{(k)} \ge 0$ by \Cref{assump:separable}, we have 
    \begin{align}
        \label{eq: J1 acc}
        J_1 &:= \frac{1}{m} \sum_{k=1}^m a_k y \Bigl (\phi'(\xB^\top \wB^{(k)}) \xB^\top \uB_1^{(k)} + \phi(\xB^\top \wB^{(k)}) - \phi'(\xB^\top \wB^{(k)}) \xB^\top  \wB^{(k)}\Bigr ) \notag \\ 
        &\ge  \frac{1}{m} \sum_{k=1}^m a_k \alpha  y \xB^\top \uB_1^{(k)} +  \frac{1}{m} \sum_{k=1}^m a_k   y \Bigl ( \phi(\xB^\top \wB^{(k)}) - \phi'(\xB^\top \wB^{(k)}) \xB^\top  \wB^{(k)}\Bigr ) \notag \\ 
        &\ge  \frac{1}{m} \sum_{k=1}^m a_k^2 \alpha  \frac{\log(\gamma ^2 \eta T)+\kappa }{\alpha \gamma } y \xB^\top \wB_* - \frac{1}{m} \sum_{k=1}^m |a_k| \kappa \notag \\
        &\qquad \explain{since $|\phi(\xB^\top \wB^{(k)}) - \phi'(\xB^\top \wB^{(k)}) \xB^\top  \wB^{(k)}|\le \kappa$ by \Cref{assump: activation:near-homogeneous}}\notag\\ 
        & \ge \log(\gamma ^2 \eta T) +\kappa - \kappa  \notag \\
        &\qquad \explain{since $y\xB^\top \wB_* \ge \gamma$ and  $\sum_{k=1}^m a_k^2 = m$} \notag \\ 
        &  = \log (\gamma ^2 \eta T).
    \end{align}
    Plugging in  $J_2 = y f(\wB; \xB)$ and \eqref{eq: J1 acc} into \eqref{eq: I1 acc}, we get
    \begin{align*}
            I_1(\wB) &= \langle \nabla L(\wB), \uB_1 - \wB \rangle
            = \ell'\big(yf(\wB;\xB)\big)(J_1 -J_2) \\
            &\le   \ell' \bigl ( y f(\wB; \xB)   \bigr ) \Bigl[  \log(\gamma ^2 \eta T) - y f(\wB; \xB) \Bigr] && \explain{since $\ell'<0$} \\ 
            & \le  \ell(\log(\gamma ^2 \eta T)) - \ell( y f(\wB; \xB)) && \explain{since $\ell$ is convex} \\
            & \le \frac{1}{\gamma ^2 \eta T} - L(\wB). 
    \end{align*}
    where in the last inequality, we use $\ell(x) \le \exp(-x)$ and we only consider a single data point. This completes the proof of \Cref{lem: Bound of I1 in EoS}.
\end{proof}

\begin{lemma} [A bound on $I_2(\wB)$ in EoS] 
\label{lem: Bound the variance in EoS}
For $\uB_2$ defined in \eqref{eq:ub2 acc},  for every $\wB$, 
\[
    I_2(\wB):=2\langle \nabla L(\wB), \uB_2 \rangle + m \eta \|\nabla L(\wB)\|^2 \le 0.  
\]
\end{lemma}
\begin{proof}[Proof of \Cref{lem: Bound the variance in EoS}]

For simplicity, we define 
\[g_i(\wB^{(j)}) \coloneqq \ell'(y_i f(\wB; \xB_i))  \phi'(\xB_i^\top \wB^{(j)}) .\]
Note that $-1\le \ell'(\cdot)\le 0$ and $
0<\alpha \le \phi'(\cdot) \le 1$, we have 
\[-1\le g_i(\wB^{(j)})\le 0.\]
Under this notation, we have 
\begin{align*}
\frac{\partial L(\wB)}{\partial \wB_i} 
&= \frac{1 }{n}\sum_{i=1}^n \ell'\big(y_i f(\wB; \xB_i)\big)  y_i   a_j  m^{-1}   \phi'\big(\xB_i^\top \wB^{(j)}\big)  \xB_i \\
&= \frac{1}{n}\sum_{i=1}^n g_i\big(\wB^{(j)}\big) a_j m^{-1} y_i \xB_i.
\end{align*}
So we have 
\begin{align*}
    I_2(\wB) &:= 2\langle \nabla L(\wB), \uB_2 \rangle + m\eta \|\nabla L(\wB)\|^2 \\ 
    &= \frac{1}{m}\sum_{j=1}^m \Biggl[ \frac{2}{n}\sum_{i=1}^n g_i(\wB^{(j)}) a_j y_i\cdot \xB_i^\top \uB_2^{(j)}  + \eta \biggl\| \frac{1}{n}\sum_{i=1}^n g_i(\wB^{(j)}) a_jy_i \xB_i  \biggr\|^2 \Biggr] .
\end{align*}
For the term inside the bracket, we have
\begin{align*}
    \lefteqn{ \frac{2}{n}\sum_{i=1}^n g_i(\wB^{(j)}) a_j y_i\cdot \xB_i^\top \uB_2^{(j)}  + \eta \bigg\| \frac{1}{n}\sum_{i=1}^n g_i(\wB^{(j)}) a_jy_i \xB_i  \bigg\|^2} \\ 
    = &\frac{2}{n}\sum_{i=1}^ng_i(\wB^{(j)}) a_j  y_i \cdot \xB_i^\top\frac{\eta }{2\gamma} a_j  \wB_*    + \eta \bigg\| \frac{1}{n}\sum_{i=1}^n g_i(\wB^{(j)}) a_jy_i \xB_i \bigg\|^2 &&\explain{since $\uB_2^{(j)} := \frac{\eta  a_j}{2 \gamma}  \wB_*$ by \eqref{eq:ub2 acc}}\\ 
    \le & \frac{2}{n}\sum_{i=1}^ng_i(\wB^{(j)}) a_j ^2 \frac{\eta}{2\gamma}  \gamma   + \eta \bigg\| \frac{1}{n}\sum_{i=1}^n g_i(\wB^{(j)}) a_jy_i \xB_i  \bigg\|^2 &&\explain{since $g_i(\cdot)\le 0$ and  $y_ix_i^\top \wB_* \ge \gamma $}\\ 
    = &\eta  \bigg( \frac{1}{n}\sum_{i=1}^n g_i(\wB^{(j)})  +  \bigg\| \frac{1}{n}\sum_{i=1}^n g_i(\wB^{(j)}) y_i \xB_i  \bigg\|^2\bigg) && \explain{since $a_j^2=1$}\\
    \le &\eta  \bigg( \frac{1}{n}\sum_{i=1}^n g_i(\wB^{(j)})  +   \frac{1}{n}\sum_{i=1}^n g_i^2(\wB^{(j)})\bigg) && \explain{since $|g_i(\cdot)|\le 1$ and $\|y\xB\|\le 1$}\\
    \le &0. && \explain{since $ -1\le g_i(\cdot)\le 0$}
\end{align*}
Hence, we prove that $I_2(\wB)\le 0$. This completes the proof of \Cref{lem: Bound the variance in EoS}.
\end{proof}

\begin{theorem}
[A split optimization bound]
\label{thm: A split optimization bound}
For every $\eta >0$ and $\uB= \uB_1 + \uB_2$ such that 
\[
    \uB_1 := \begin{pmatrix}
        \uB_1^{(1)} \\
        \vdots \\
        \uB_1^{(m)}
    \end{pmatrix}, \ \
    \text{with}\ \  
    \uB_1^{(j)} := a_j \frac{\log(\gamma^2 \eta t) + \kappa }{\alpha \gamma } \cdot \wB_*, \quad j=1,\ldots ,m,
\]
and 
\[
\uB_2 := \begin{pmatrix} 
        \uB_2^{(1)}\\ 
        \vdots\\ 
        \uB_2^{(m)}
        \end{pmatrix}, \ \ 
        \text{with}\ \  \uB_2^{(j)} := a_j \frac{\eta }{ 2\gamma } \cdot \wB_*, \quad j=1,\ldots ,m.
\] 
we have: 
\[
    \frac{\|\wB_T-\uB\|^2}{2 m \eta T}+\frac{1}{T} \sum_{k=0}^{T-1} L\bigl (\wB^{(k)}\bigr ) \leq \frac{1+8\log ^2(\gamma^2 \eta T)/\alpha ^2+ 8\kappa^2/\alpha ^2+\eta^2 }{\gamma^2 \eta T} + \frac{\|\wB_0\|^2}{ m\eta T}, 
\]
for all $T$. 

\end{theorem}
\begin{proof}[Proof of \Cref{thm: A split optimization bound}]
By \Cref{lem: Bound of I1 in EoS} and \Cref{lem: Bound the variance in EoS}, we have 
\[
    \begin{aligned}
        \|\wB_{t+1}-\uB\|^2 
        & =\|\wB_t-\uB\|^2+ 2m\eta I_1(\wB_t) +\eta m I_2(\wB_t) \\
        & \leq\|\wB_t-\uB\|^2+2 m\eta I_1(\wB_t) \\
        & \leq\|\wB_t-\uB\|^2+  2 m\eta \bigg ( \frac{1}{\gamma^2 \eta T}  - L(\wB_t) \bigg ).
        \end{aligned}
\]
Telescoping the sum, we get
$$
\frac{\|\wB_T-\uB\|^2}{2 m\eta}+\sum_{t=0}^{T-1} L(\wB_t) \leq 1 +\frac{\|\wB_0-\uB\|^2}{2 m \eta}. 
$$
By \Cref{eq:ub1 acc}and \Cref{eq:ub2 acc}, we have 
\begin{align*}
        \|\wB_0 - \uB\|^2 &\le 2 \|\wB_0\|^2 + 2\|\uB\|^2  \\ 
        &\le2 \|\wB_0\|_2^2 +4\|\uB_1\|^2 + 4\|\uB_2\|^2 \\ 
        & = 2\|\wB_0\|_2^2 + \frac{8m\log (\gamma^2 \eta T)^2 + 8m\kappa^2}{\alpha ^2 \gamma ^2 } + \frac{m\eta^2 }{\gamma^2},
\end{align*}
which implies that
\[
    \frac{\|\wB_T-\uB\|^2}{2 m \eta t}+\frac{1}{T} \sum_{k=0}^{T-1} L\bigl (\wB^{(k)}\bigr ) \leq \frac{1+8\log ^2(\gamma^2 \eta T)/\alpha ^2+ 8\kappa^2/\alpha ^2+\eta^2 }{\gamma^2 \eta T} + \frac{\|\wB_0\|^2}{ m\eta T}. 
\]
We complete the proof of \Cref{thm: A split optimization bound}.
\end{proof}

\subsection{Proof of Theorem~\ref{thm: Bound of EoS phase}}\label{sec:thm-eos}
\begin{proof}[Proof of Theorem~\ref{thm: Bound of EoS phase}]
By \Cref{thm: A split optimization bound}, we have
\[
    \frac{1}{T} \sum_{k=0}^{T-1} L\bigl (\wB^{(k)}\bigr ) \leq \frac{1+8\log ^2(\gamma^2 \eta T)/\alpha ^2+ 8\kappa^2/\alpha ^2+\eta^2 }{\gamma^2 \eta T} + \frac{\|\wB_0\|^2}{ m\eta T}.  
\]
This completes the proof.
\end{proof} 

\section{Phase Transition Analysis} \label{sec:phase}
In this section, we will analyze the phase transition. In detail, we follow the idea of \citep{wu2024large} and apply the perceptron argument \citep{novikoff1962convergence} to locate the phase transition time. Compare to the previous EoS phase analysis, we need an extra assumption on the smoothness of the activation function, which is the \Cref{assump:activation:smooth}. 

To proceed, let us define the following quantities for the GD process: 
\[
    G(\wB): =\frac{1}{n} \sum_{i=1}^n \frac{1}{1+\exp  \big( y_i f(\wB; \xB_i) \big)}, \quad F(\wB) :=\frac{1}{n} \sum_{i=1}^n \exp \big(-y_i f(\wB; \xB_i)\big). 
\]

Due to the self-boundedness of the logistic function, we can show that $G(\wB), L(\wB), F(\wB)$ are equivalent in the following sense. 

\begin{lemma}
[Equivalence of $G,L,F$]
\label{lem: Bounds item1for GLF}
\begin{itemize}
    \item []
    \item [1.] $G(\wB) \le L(\wB) \le F(\wB)$. 
    \item [2.] $\alpha  \gamma G(\wB)\le \sqrt{m}  \|\nabla L(\wB)\| \le G(\wB)$.
    \item [3.] If $G(\wB) \le \frac{1}{2n}$, then $F(\wB) \le 2G(\wB)$. 
\end{itemize}
\end{lemma}
\begin{proof}[Proof of \Cref{lem: Bounds item1for GLF}]
The first claim is by the property of the logistic loss.
For the second one, 
\begin{align*}
        \|\nabla L(\wB)\|^2 &=  \sum_{j=1}^m \Biggl \|\frac{1}{n} \sum_{i=1}^n  \ell'(y_i f(\wB; \xB_i)) \cdot y_i \cdot  a_j m^{-1}  \phi(\xB_i^\top \wB^{(j)}) \xB_i \Biggr\|_2^2 \\ 
        &\le \sum_{j=1}^m \Biggl ( \frac{1}{n} \sum_{i=1}^n  \ell'(y_i f(\wB; \xB_i))  \cdot   m^{-1}  \Biggr )^2 &\explain{since $\|y_i a_j \phi(\xB_i^\top \wB^{(j)}) \xB_i \| \le 1$ } &\\  
        &= \frac{1}{m} G^2(\wB). 
\end{align*}
Besides, we have 
\begin{align*}
   \sqrt{m}  \| \nabla L(\wB)\| &\ge  \langle -\nabla L (\wB), \bar \wB_* \rangle &\explain{ since $\|\bar \wB_*\| \le \sqrt{m}$   }  \\ 
    & = -\frac{1}{nm} \sum_{i=1}^n \sum_{j=1}^m \ell'(y_if(\wB; \xB_i)) y_i \phi'(\xB_i^\top \wB_*) \xB_i^\top \wB_* \\ 
    &\ge \alpha \gamma \frac{1}{n} \sum_{i=1}^n \frac{1}{1+\exp \big(y_i f(\wB; \xB_i)\big)} &\explain{since $\phi^\prime \ge \alpha $ and $y_i \xB_i^\top  \wB^* \ge \gamma$ } \\
    &= \alpha \gamma G(\wB).
\end{align*}
For the third claim,
by the assumption, we have 
\[
    \frac{1}{n} \cdot \frac{1}{1+ \exp \big (y_i f(\wB; \xB_i) \big)} \le G(\wB) \le \frac{1}{2n},
\]
which implies that 
\[
    y_i f(\wB; \xB_i) \ge 0, \quad \forall i \in [n]. 
\]
Therefore, 
\[
    G(\wB) = \frac{1}{n} \sum_{i=1}^n \frac{1}{1+\exp \bigl ( y_i f(\wB; \xB_i)\bigr )} \ge \frac{1}{n} \sum_{i=1}^n \frac{1}{2\exp \bigl ( y_i f(\wB; \xB_i)\bigr )} = \frac{1}{2} F(\wB). 
\]
We complete the proof of \Cref{lem: Bounds item1for GLF}.
\end{proof}

The key ingredient of the phase transition analysis is the following lemma. The main idea is to consider the gradient potential $G(\wB)$ instead of the loss function $L(\wB)$ in EoS phase. And this will decrease the order of the bound of phase transition time from $\tilde O(\eta^2)$ to $\tilde O(\eta)$.

\begin{lemma}
[A bound of $\|\wB_t\|$]
\label{lem: A bound of wBt}
For every $\eta$, we have 
\[
    \|\wB_t\|\le \sqrt{m} \cdot \frac{2+8\log (\gamma^2 \eta t)/\alpha+ 8\kappa/\alpha+4\eta }{\gamma} + 2\|\wB_0\|.
\]
\end{lemma}
\begin{proof}[Proof of \Cref{lem: A bound of wBt}]
By \Cref{thm: A split optimization bound}, we have 
\[
    \frac{\|\wB_t-\uB\|^2}{2 m \eta t}\le   \frac{\|\wB_t-\uB\|^2}{2 m \eta t}+\frac{1}{t} \sum_{k=0}^{t-1} L\bigl (\wB^{(k)}\bigr ) \leq \frac{1+8\log ^2(\gamma^2 \eta t)/\alpha ^2+ 8\kappa^2/\alpha ^2+\eta^2 }{\gamma^2 \eta t} + \frac{\|\wB_0\|^2}{ m\eta t}. 
\]

Besides, we know that 
\[
    \|\uB\|^2 \le 2\|\uB_1\|^2 + 2\|\uB_2\|^2=  \frac{4m\log (\gamma^2 \eta t)^2 + 4m\kappa^2}{\alpha ^2 \gamma ^2 } + \frac{m\eta^2 }{2\gamma^2}. 
\]
Combining them, we have 
\[
    \|\wB_t\|^2 \le 2\|\wB_t - \uB\|^2 + 2\|\uB\|^2 \le m\cdot \frac{2+24\log ^2(\gamma^2 \eta t)/\alpha ^2+ 24\kappa^2/\alpha ^2+3\eta^2 }{\gamma^2} + 2\|\wB_0\|^2. 
\]
Hence, we can get a bound for $\|\wB_t\|$. 
\[
    \|\wB_t\|\le \sqrt{m} \cdot \frac{2+8\log (\gamma^2 \eta t)/\alpha+ 8\kappa/\alpha+4\eta }{\gamma} + 2\|\wB_0\|.
\]
Now we have completed the proof of \Cref{lem: A bound of wBt}.
\end{proof}

\begin{lemma}
[Gradient potential bound in the EoS phase]
\label{lem: Gradient potential bound in EoS}
For every $\eta$, we have  
\[
    \frac{1}{t} \sum_{k=0}^{t-1} G(\wB^{(k)}) \le \frac{\langle \wB_t, \bar\wB_* \rangle - \langle \wB_0, \bar \wB_* \rangle  }{m \alpha \gamma \eta t} \le \frac{\sqrt{m} \|\wB_t\|- \langle \wB_0, \bar \wB_* \rangle}{m \alpha \gamma \eta t}, \quad t\ge 1.   
\]
Additionally,  we have 
\[
    \frac{1}{t} \sum_{k=0}^{t-1} G(\wB^{(k)}) \le \frac{2+8 \log \big (\gamma^2 \eta t\big)/\alpha + 8\kappa/ \alpha   +4\eta  }{\alpha \gamma^2 \eta t} + \frac{3 \|\wB_0\|}{\alpha\gamma \eta t}, \quad t \ge 1.   
\]
\end{lemma}
\begin{proof}[Proof of \Cref{lem: Gradient potential bound in EoS}]
This is from the perceptron argument \citep{novikoff1962convergence}. Specifically, 
\begin{align*}
    \langle \wB_{t+1},  \bar \wB_* \rangle 
    &= \langle \wB_t, \bar  \wB_* \rangle - m\eta \langle \nabla L(\wB_t), \bar \wB_* \rangle\\ 
    &= \langle \wB_t, \bar \wB_* \rangle - \eta  \sum_{i=1}^{n}\sum_{k=1}^{m} a_k^{2} \ell'(y_i f(\wB_t;\xB_i))y_i \phi(\xB_i^\top \wB_{t}^{(k)}) \langle \xB_i, \wB_* \rangle      \\ 
    & \ge \langle \wB_t, \bar \wB_* \rangle - \eta  \sum_{i=1}^{n}\sum_{k=1}^{m} a_k^{2} \ell'(y_i f(\wB_t;\xB_i))\alpha  \gamma \\ 
    & \ge \langle \wB_t, \bar \wB_* \rangle + m \alpha \gamma \eta G(\wB_t).  
\end{align*}
Telescoping the sum, we have 
\begin{align*}
        \frac{1}{t} \sum_{k=0}^{t-1} G(\wB^{(k)}) &\le \frac{\langle \wB_t, \bar \wB_* \rangle - \langle \wB_0,\bar  \wB_* \rangle  }{m\alpha \gamma \eta t}  \\ 
        &\le   \frac{\sqrt{m} \|\wB_t\| - \langle \wB_0, \bar \wB_* \rangle   }{m\alpha \gamma \eta t} \\ 
        &\le \frac{2+8 \log \big (\gamma^2 \eta t\big)/\alpha + 8\kappa/ \alpha   +4\eta  }{\alpha \gamma^2 \eta t} + \frac{3 \|\wB_0\|}{\sqrt{m} \alpha\gamma \eta t}. &&\explain{ by \Cref{lem: A bound of wBt}} 
\end{align*}
We have completed the proof of \Cref{lem: Gradient potential bound in EoS}.
\end{proof}

Besides, we can make use of the equivalence between $G$ and $L$ to get a bound for the loss function which is independent of the initial margin at $s$. 
\begin{lemma}
[A risk bound in the stable phase for Two-layer NN]
\label{lem: Risk bound stable 2nn}
Suppose that there exists a time $s$ such that
\[
    L(\wB_s) \le \min \bigg\{ \frac{1}{\eta(4+2\tilde \beta)}, \frac{1}{2 e^{\kappa+2}n} \bigg\}. 
\] 
Then for every $t\ge s+1$, we have 
\[
    L(\wB_t) \le \frac{2}{(t-s)\alpha ^2 \gamma ^2}. 
\] 
\end{lemma}
\begin{proof}[Proof of \Cref{lem: Risk bound stable 2nn}]
By \Cref{lem: Decrease of Lt} and $f(x)$ is $\frac{1}{\sqrt{m} }$ Lipschitz and $\frac{\tilde \beta}{m}$ smooth, we have 
\[
    L_{k+1} \le L_k - m\eta(1 - (2 + \tilde\beta) \eta L(\wB_k) ) \|\nabla L_t\|^2. 
\]
By \Cref{lem: Bounds item1for GLF} and $L_t \le \frac{1}{\eta(4+2\tilde\beta)}$, we have 
\[
    L_{k+1} \le L_k -\frac{ \alpha^2 \gamma ^2 }{2} L_k^2. 
\]
Multiplying $\frac{1}{L_k^2}$ in both sides, we have 
\begin{align*}
    \frac{\alpha ^2 \gamma ^2 }{2} \le \frac{L_t - L_{k+1}}{L_k^2 } \le \frac{1}{L_{k+1}} - \frac{1}{L_k}. 
\end{align*}
Taking summation for $k=s,\ldots ,t-1$, we have 
\[
    \frac{1}{L_t} > \frac{1}{L_t} - \frac{1}{L_s} \ge \frac{(t-s) \alpha ^2 \gamma ^2}{2} \implies L_t \le \frac{2}{(t-s) \alpha ^2 \gamma ^2}.  
\]
This completes the proof  of \Cref{lem: Risk bound stable 2nn}.
\end{proof}

At last, we will use the bound for the gradient potential to get an upper bound for the phase transition time.

\subsection{Proof of Theorem~\ref{thm:phasetran}} \label{sec:thm:tran}
\begin{proof}[Proof of Theorem~\ref{thm:phasetran}]
Applying \Cref{lem: Gradient potential bound in EoS}, we have 
\begin{align*}
    \frac{1}{\tau} \sum_{k=0}^{\tau-1} G(\wB^{(k)}) &\le \frac{2+8 \log \big (\gamma^2 \eta \tau\big)/\alpha + 8\kappa/ \alpha   +4\eta  }{\alpha \gamma^2 \eta \tau} + \frac{3 \|\wB_0\|}{\sqrt{m} \alpha\gamma \eta \tau} \\ 
    &\le \frac{2+8\kappa/\alpha +8 \log  (\gamma^2 \tau )/\alpha   +(4+8/\alpha )\eta }{\alpha \gamma^2 \eta \tau} + \frac{3 \|\wB_0\|}{\sqrt{m} \alpha\gamma \eta \tau} &&\explain{ since $\log(\eta) \le \eta$}.  
\end{align*}
Let $c_1 = 4e^{\kappa+2}, c_2 = (8+4\tilde\beta)$. Note that we have 
\begin{align*}
\frac{2+8\kappa/\alpha }{\alpha  \gamma^2\eta \tau }\le \frac{1}{4(c_1n  + c_2\eta)} &\text{ if } \gamma^2 \tau \ge 4(2+8\kappa) \frac{c_2\eta + c_1n}{\eta \alpha ^2} \\ 
    \frac{ 8 \log  (\gamma^2 \tau )/\alpha }{\alpha  \gamma^2 \eta\tau }\le \frac{1}{4(c_1n  + c_2\eta)} &\text{ if } \gamma^2 \tau \ge 128 \frac{c_2\eta +c_1 n}{\eta \alpha ^2}  \log \frac{c_2\eta+c_1n}{\eta}, \explain{since \Cref{lem:thresh bound}} \\ 
    \frac{(4+8/\alpha)\eta }{\alpha \gamma ^2 \eta \tau}\le \frac{1}{4(c_1n  + c_2\eta)} &\text{ if } \gamma^2 \tau \ge \frac{48}{\alpha ^2} (c_2\eta + c_1n), \\ 
    \frac{3 \|\wB_0\|}{m\alpha \gamma  \eta \tau}\le \frac{1}{4(c_1n  + c_2\eta)} &\text{ if } \gamma \tau \ge \frac{12}{\alpha } \frac{(c_2\eta+ c_1n)}{\eta} \cdot \frac{\|\wB_0\|}{\sqrt{m}  }
\end{align*}
and that the two conditions are satisfied because
\begin{align*}
        \gamma^2 \tau &\coloneqq  \frac{128(1+4\kappa)}{ \alpha ^2} \max \bigg \{ c_2\eta, c_1n, e, \frac{c_2\eta+c_1n}{\eta} \log \frac{c_2\eta+c_1n}{\eta}, \frac{(c_2\eta+c_1n)\|\wB_0\|}{\eta \sqrt{m}}\bigg \} \\
        &\ge \max \bigg \{  4(2+8\kappa) \frac{c_2\eta + c_1n}{\eta \alpha ^2}, 128 \frac{c_2\eta +c_1 n}{\eta \alpha ^2}  \log \frac{c_2\eta+c_1n}{\eta} ,\frac{48}{\alpha ^2} (c_2\eta + c_1n),  \frac{12}{\alpha } \frac{(c_2\eta+ c_1n)}{\eta} \cdot \frac{\|\wB_0\|}{\sqrt{m} }\bigg \}. 
\end{align*}

So there exits $s\le \tau$ such that 
\[
    G(\wB_s) \le \min \bigg\{\frac{1}{e^{\kappa+2} 4 n}, \frac{1}{{\eta}(8+4\tilde \beta)}\bigg\}
\]
Then we have $L(\wB_s) \le F(\wB_s) \le 2 G(\wB_s) \le \Big\{\frac{1}{e^{\kappa+2} 2 n}, \frac{1}{{\eta}(4+2\tilde \beta)}\Big\}$. 
We complete the proof of Theorem~\ref{thm:phasetran}.
\end{proof}

    

\subsection{Proof of Corollary~\ref{cor:acceleration}}\label{sec:cor-acc}
\begin{proof}[Proof of Corollary~\ref{cor:acceleration}]
The main idea is to show that $\tau \le \frac{T}{2}$. Note that by \Cref{thm:phasetran}, we have 
\[
    \tau =\frac{128(1+4\kappa)}{ \alpha ^2} \max \bigg \{ c_2\eta, c_1n, e, \frac{c_2\eta+c_1n}{\eta} \log \frac{c_2\eta+c_1n}{\eta}, \frac{(c_2\eta+c_1n)}{\eta}\cdot \frac{\|\wB_0\|}{\sqrt{m}}\bigg \},
\]
in which expression $c_1 = 4e^{\kappa+2}$ and $c_2 = (8+4\tilde\beta)$.  We can verify that, 
\begin{align*}
    \frac{128(1+4\kappa)}{ \alpha ^2}c_2\eta &= \frac{128(1+4\kappa)}{ \alpha ^2}c_2 \cdot \frac{\alpha ^2 \gamma ^2}{256(1+4\kappa)c_2 }T = \frac{T}{2},\\ 
    \frac{128(1+4\kappa) c_1 n}{ \alpha ^2} \le \frac{T}{2}. 
\end{align*}
Furthermore, we have $n\le \frac{\alpha ^2 \gamma ^2 T}{256(1+4\kappa) c_1}$. Hence,  
\[
    \frac{c_2 \eta  + c_1 n}{\eta } = \frac{\frac{\alpha ^2 \gamma ^2 T}{256(1+4\kappa)} + c_1n}{\frac{\alpha ^2 \gamma ^2 T}{256(1+4\kappa)c_2}} \le \frac{2 \cdot \frac{\alpha ^2 \gamma ^2 T}{256(1+4\kappa)}}{\frac{\alpha ^2 \gamma ^2 T}{256(1+4\kappa)c_2}} \le 2c_2.  
\]
We get that: 
\begin{align*}
    \frac{128(1+4\kappa)}{ \alpha ^2} \cdot \frac{c_2\eta+c_1n}{\eta} \log \frac{c_2\eta+c_1n}{\eta} &\le 2 \frac{128(1+4\kappa)}{ \alpha ^2} c_2 \ln (2c_2) \le \frac{128(1+4\kappa)}{ \alpha ^2}  4c_2^2\le \frac{T}{2}, \\ 
    \frac{128(1+4\kappa)}{ \alpha ^2} \cdot \frac{(c_2\eta+c_1n)}{\eta}\cdot \frac{\|\wB_0\|}{\sqrt{m}} &\le\frac{128(1+4\kappa)}{ \alpha ^2} \cdot  2c_2 \frac{\|\wB_0\|}{\sqrt{m} } \le \frac{T}{2}. 
\end{align*}
Hence, we have $\tau \le \frac{T}{2}$. Applying \Cref{thm:phasetran}, we have 
\[
    L(\wB_T) \le \frac{2}{\alpha ^2 \gamma ^2  \eta(T-\tau)} \le \frac{4}{\alpha ^2 \gamma ^2 \eta T} \le \frac{2048(1+4\kappa) c_2}{\alpha ^4 \gamma ^4T^2} = \Ocal(1/T^2). 
\]
We have completed the proof of Corollary~\ref{cor:acceleration}. 
\end{proof}

\subsection{Proof of Theorem~\ref{thm: Lower bound in the classical regime}}\label{sec:thm:classical}
\begin{proof}[Proof of Theorem~\ref{thm: Lower bound in the classical regime}]
The main idea is to construct an upper bound of $\eta$ and apply the analysis in \Cref{thm: Implicit bias and bound of stable phase}.  Note that give $\wB_0=0$, we have 
\[
    f(\wB_0; \xB_i) = \frac{1}{m}\sum_{k=1}^m a_k \phi(\xB_i^\top \wB^{(k)}_0) = s_a \phi(0), 
\]
where $s_a = \sum_{k=1}^m a_k/m$. Therefore, 
\begin{align*}
        [\nabla L(\wB_0)]^{(k)}  &= \frac{1}{2} \ell^\prime (s_a \phi(0)) \cdot \frac{a_k}{m} \phi'(0) \xB_1 + \frac{1}{2} \ell^\prime (s_a \phi(0)) \cdot  \frac{a_k}{m}  \phi'(0) \xB_2  \\ 
        & = \frac{a_k}{m} \ell^\prime (s_a \phi(0)) \phi'(0) \frac{\xB_1 + \xB_2}{2}\\ 
        & = \frac{a_k}{m} \ell^\prime (s_a \phi(0)) \phi'(0) (\gamma, \frac{\sqrt{1-\gamma ^2} }{4}). 
\end{align*}
Let $\bar \xB \coloneqq \frac{1}{m}\ell^\prime (s_a \phi(0)) \phi'(0) (\gamma, \frac{\sqrt{1-\gamma ^2} }{4})$, we have 
\[
    \wB_1^{(k)} = 0- \eta \nabla [L(\wB_0)]^{(k)}= - \eta a_k \bar \xB.  
\]
Therefore,  
\[
    f(\wB_1;\xB_i) = \frac{1}{m} \sum_{k=1}^m a_k \phi(-\xB_i^\top (\eta a_k \bar \xB)). 
\]
We can notice that $-\xB_1 ^\top \bar \xB <0$ and $-\xB_2^\top \bar \xB >0$, when $\gamma \le 0.1$. Furthermore, we have 
\begin{align*}
        f(\wB_1; \xB_1) &= \frac{1}{m} \sum_{a_k=1} \phi( -\xB_1^\top (\eta \bar \xB)) + \frac{1}{m} \sum_{a_k=-1} -\phi( \xB_1^\top (\eta \bar \xB))\\ 
        & = \frac{1}{m} \sum_{a_k=1} [\phi(0)  -\xB_1^\top (\eta \bar \xB) \phi^\prime (\epsilon_1) ]  + \frac{1}{m} \sum_{a_k=-1} [-\phi(0) - \xB_1^\top (\eta \bar \xB) \phi^\prime (\epsilon_2) ] \\ 
        &= s_a \phi(0) - \eta \xB_1^\top \bar \xB \frac{1}{m} [\sum_{a_k=1}     \phi^\prime (\epsilon_1)+ \sum_{a_k=-1} \phi^\prime (\epsilon_2) ] \\
        &\le s_a \phi(0) - \eta \xB_1^\top \bar \xB \alpha  && \explain{ $\phi^\prime (\epsilon _i) \ge \alpha $}.    
\end{align*}
Note that 
\[
   \frac{1}{2}\ell(s_a\phi(0) - \eta \xB_1^\top \bar \xB \alpha) \le   \frac{1}{2}\ell(f(\wB_1;\xB_1))\le L(\wB_1) \le L(\wB_0) = \ell(s_a\phi(0)). 
\]
We apply \Cref{lem:class} to get
\[
    \eta \le \frac{|s_a \phi(0)| + \ln 3}{\xB_1^\top \bar \xB \alpha }. 
\]
We use $c_3\coloneqq  \frac{|s_a \phi(0)| + \ln 3}{\xB_1^\top \bar \xB \alpha }$. Now we know $\eta \le c_3$.  Furthermore, notice that 
\[
    \|\nabla L(\wB_t)\| \le  L_t \le L_0. 
\] 
We get that $\|\wB_{t+1} - \wB_t\| \le \eta L_0 \le c_3 L_0$. Hence, 
\[
    |f(\wB_{t+1}; \xB_i) - f(\wB_t; \xB_i)| \le c_3 L_0. 
\]
Assume that $l_b = \min \big\{\frac{1}{e^{\kappa+2} 4 n}, \frac{1}{{\eta}\left(8+4 \tilde \beta\right)}\big\}$ and 
\[
    L_{s-1} \ge l_b, \quad L_s \le l_b. 
\]
We know that 
\[
    l_b \ge \min \bigg\{\frac{1}{e^{\kappa+2}4n}, \frac{1}{c_3(8+4\tilde \beta)}\bigg\}\eqqcolon l_c.
\]
We want to show that there is an lower bound for $L_s$. Now that 
\[
    L_s = \frac{1}{2} \Big[  \ell(f(\wB_s;\xB_1)) + \ell(f(\wB_s;\xB_2))\Big]. 
\] 
Applying \Cref{lem:lowerbound l}, we can get that 
\[
    L_s \ge \exp(-c_3L_0) L_{s-1} \ge \exp(-c_3 L_0) l_b.
\]
Recall that by \Cref{lem: Order of Loss in general model}, we have 
\[
    L_t \ge \frac{1}{ \frac{1}{ L_s} + 3\tilde\eta \rho^2(t-s)},\quad t \ge s.  
\]
Combine this with $\rho=\frac{1}{\sqrt{m} }, \tilde \eta = \eta m$ and $L_s \ge \exp(-c_3 L_0) l_b$ and we get 
\[
    L_t \ge \frac{1}{\frac{\exp(c_3 L_0)}{l_b} + 3\eta (t-s)}, \quad t\ge s. 
\]
Note that when $t\le s$, $L_t \ge l_b$. We can get a lower bound for $L_t$ by 
\[
    L_t \ge \frac{1}{\frac{\exp(c_3 L_0)}{l_b}t + 3\eta t} \ge \frac{1}{\frac{\exp(c_3 L_0)}{l_b}t + 3c_3t} \ge \frac{1}{\frac{\exp(c_3 L_0)}{l_c}t + 3c_3t}= \frac{c_4}{t},  
\]
where $c_4 = \frac{1}{\frac{\exp(c_3 L_0)}{l_c} + 3c_3}$ depends only on $\{a_j\}_{j=1}^m, \phi(0), \kappa, \tilde \beta$ and  $n$. Now we have completed the proof of Theorem~\ref{thm: Lower bound in the classical regime}.
\end{proof}

\section{Scaling and Homogenous Error} \label{sec:scaling}

In this section, we consider different scaling of two-layer networks. We add a scaling factor $b$ into the model, i.e.,
\[
    f(\wB;\xB) = \frac{b}{m} \sum_{j=1}^m a_j \phi( \xB^\top \wB^{(j)}). 
\]
We will show that given a limited computation budget $T$ (total iterations), larger $b$ and a corresponding best $\tilde \eta =\eta\cdot m$ will achieve the same best rate as $b=1$, i.e., $O(1/T^2)$. While for smaller $b$, the rate is $O(b^{-3}/T^2)$. Before we present the analysis, here are the bounds with $b$ and $\tilde \eta = m \cdot \eta$ following the process of \Cref{lem: Gradient potential bound in EoS}: 
\begin{align*}
    \frac{1}{t} \sum_{k=0}^{t-1} L (\wB_k )  \leq \frac{1+8\log ^2(\gamma^2 \eta t)/(\alpha ^2 b^2)+ 8\kappa^2/\alpha ^2+\eta^2 b^2 }{\gamma^2 \eta t} + \frac{\|\wB_0\|^2}{ m\eta  t}, \\ 
    \frac{1}{t} \sum_{k=0}^{t-1} G (\wB_k )  \leq \frac{2+8\log (\gamma^2 \eta t)/(\alpha b)+ 8\kappa /\alpha+2\eta b }{\alpha \gamma^2 b \eta t} + \frac{3\|\wB_0\|}{ \sqrt{m} \eta b t}.  
\end{align*}

\paragraph{Case when $b\ge 1$.}Given the previous bounds, we have the following results following the idea in \Cref{sec:phase}: 
\begin{itemize}
    \item Gradient potential bound: $G(\wB_t) \le \frac{C}{t}$ for all $t\ge 0$, 
    \item Phase transition threshold: $G(\wB_s) \le \min \Big\{ 1/4e^{\kappa+2}n, 1/\eta(8\rho^2 b^2 + 4\tilde\beta b)\Big\}$,
    \item Stable phase bound: $L(\wB_t) \le \frac{2}{C b^2 \eta(t-s)}$, 
\end{itemize}
where $C$ depends on $\alpha ,\gamma$.  Combine the first two arguments and assume $\eta(8\rho^2b^2 + 4\tilde\beta b) \ge 4e^{\kappa+2}n$. We get $s \le C \eta (8\rho^2b^2 + 4\tilde \beta b)$. Plug this into the third bound. We have 
\[
    L(\wB_T)  \le \frac{2}{Cb^2 \eta(T-C\eta (8\rho^2 b^2 + 4\tilde\beta b))}. 
\]
It's obvious that the best $\eta = \frac{T}{16\rho^2 b^2C + 8\tilde\beta bC}$. Hence, 
\[
L(\wB_T) \le  \frac{8(8\rho^2 b^2C + 4\tilde\beta bC)}{Cb^2T^2} = \Ocal\bigg(\frac{1}{T^2}\bigg). 
\]
Then, the rate is still $O(1/T^2)$.

\paragraph{Case when $b<1$.} Similarly, we can get the following bounds: 
\begin{itemize}
    \item Gradient potential bound: $G(\wB_t) \le \frac{Cb^{-2}}{t}$ for all $t\ge 0$, 
    \item Phase transition threshold: $G(\wB_s) \le \min \Big\{ 1/4e^{\kappa+2}n, 1/\eta(8\rho^2 b^2 + 4\tilde\beta b)\Big\}$,
    \item Stable phase bound: $L(\wB_t) \le \frac{2}{C b^2 \eta(t-s)}$, 
\end{itemize}
where $C$ depends on $\alpha ,\gamma$. Without loss of generality, we can assume $\eta(8\rho^2b^2 + 4\tilde\beta b) \ge 4e^{\kappa+2}n$, since $\eta$ can be small enough. Then, we have 
\[
    s\le C\eta (8\rho^2 + 4\tilde \beta b^{-1}).  
\]
Then, we can get
\[
    L(\wB_T)  \le \frac{2}{C b^2\eta(T-C\eta (8\rho^2  + 4\tilde\beta b^{-1}))}. 
\]
It's obvious that the best $\eta = \frac{T}{16\rho^2 C + 8\tilde\beta b^{-1}C}$. Hence, 
\[
L(\wB_T) \le  \frac{8(8\rho^2 Cb^{-2} + 4\tilde\beta b^{-3}C)}{CT^2} = \Ocal\bigg(\frac{b^{-3}}{T^2}\bigg). 
\]

Combining the analysis for two cases, we observe that when $b\ge 1$, the fast loss rate is $\Ocal(1/T^2)$ given finite budget $T$. While $b<1$, the rate is $\Ocal(b^{-3}/T^2)$.
In our main results, we set $b=1$ for the mean-field scaling. Under the mean-field regime, all bounds are independent of the number of neurons since we consider the dynamics of the distributions of neurons. Alternatively, if we set $b=\sqrt{m}$, then the model becomes: 
\[
    f(\wB;\xB) = \frac{1}{\sqrt{m}} \sum_{j=1}^m a_j \phi( \xB^\top \wB^{(j)}). 
\]
The model falls into the NTK regime. The loss threshold will be related to $m$, but the loss rate is the same as that of the mean-field scaling. 


\section{Additional Proofs}
\subsection{Proof of Example \ref{eg: homo act}} \label{sec: eg homo act}
\begin{proof}[Proof of \Cref{eg: homo act}]
Recall that the two-layer neural network is defined as: 
\[
    f(\wB;\xB) = \frac{1}{m}\sum_{j=1}^m a_j \phi(\xB^T \wB^{(j)}). 
\]
We can verify that if $\phi(x)$ is $\beta$-smooth and $\rho$-Lipschitz with respect to $x$, then $f(\wB;\xB)$ is $\beta/m$-smooth and $\rho/\sqrt{m}$-Lipschitz with respect to $\wB$. This is because: 
\begin{align*}
    \grad L(\wB) &= \hat\Ebb \ell'(y f(\wB; \xB)) y \grad f(\wB; \xB),\\ 
    \grad^2 L(\wB) &= \hat\Ebb \ell''(y f(\wB; \xB)) \grad f(\wB; \xB)^{\otimes 2} + \ell'(y f(\wB; \xB)) y\grad^2 f(\wB; \xB),
\end{align*}
and that 
\begin{align*}
    \grad f(\wB;\xB) = \begin{pmatrix}
        \vdots \\
        \frac{1}{m} a_j \phi'(\xB^\top\wB^{(j)} ) \xB \\
        \vdots 
    \end{pmatrix},\quad 
    \grad^2 f(\wB;\xB) = 
    \begin{pmatrix}
        \ddots & 0 & 0 \\
      0 & \frac{1}{m} a_j \phi''(\xB^\top\wB^{(j)} ) \xB\xB^\top & 0\\
      0 & 0 & \ddots 
    \end{pmatrix}.
\end{align*}
Now, we will focus on the parameters of each activation function.
    \begin{itemize}
        \item \textbf{GELU}. $\phi(x) = x \cdot \erf(1+(x/\sqrt{2}))/2 = x\cdot F(x)$. 
\begin{align*}
    \phi^{\prime}(x) &= F(x) + x\cdot f(x),\\ 
    \phi^{\prime \prime}(x) &= 2f(x) + x \cdot f^{\prime}(x),
\end{align*} 
where $F(x), f(x)$ are the CDF and PDF of standard normal distribution.  Note that $xf(x) = \frac{x}{\sqrt{2\pi}} e^{-x^2/2}$ and $(xf(x))^{\prime} = \frac{1}{\sqrt{2\pi}} (1 - x^2) e^{-x^2/2}$. We can find the maximum of $xf(x)$ is $\frac{1}{\sqrt{2\pi}} e^{-1/2}$.  Besides, we know that $F(x), f(x) \le 1$ and $x \cdot f^{\prime}(x)\le0$. Combining them, we have $\rho  = 1+ e^{-1/2} /\sqrt{2\pi}$ and $\beta =2$. For $\kappa$, $\phi - \phi^{\prime}(x) x = -x\cdot f(x)$. So the bound of $\kappa$ is $e^{-1/2}/\sqrt{2\pi}$. 
\item \textbf{Softplus}. $\phi(x) = \log(1+e^{x})$. Therefore, 
\begin{align*}
    \phi^{\prime}(x) = \frac{e^x}{1+e^x}\le 1,\\ 
    \phi^{\prime \prime}(x) = \frac{e^x}{(1+e^x)^2} \le 1.
\end{align*}
Besides, 
\[
(\phi(x) - \phi^{\prime}(x) x)^{\prime} = \bigg( \log (1+e^x) - \frac{e^x x}{1+e^x}\bigg)^{\prime} =  - \frac{e^x x}{(1+e^x)^2}.
\]
So the maximum is $\phi(0) - \phi^{\prime}(0)0 = \log 2$. Besides, when $x>1$, $\phi(x) \ge x \ge \frac{e^x x}{1+e^x}$. When $x\to -\infty$, $\phi(x) - \phi'(x)x \to 0$. Therefore, $\kappa = \log 2$. 
\item \textbf{Sigmoid}.  $\phi(x) = 1/(1+e^{-x})$. Hence, 
\begin{align*}
    \phi^{\prime}(x) = \frac{e^{-x}}{(1+e^{-x})^2}\le 1,\\ 
    \phi^{\prime\prime }(x) = \frac{e^{-2x}-e^{-x}}{(1+e^{-x})^3}\le 1.
\end{align*}
As for $\kappa$, we know that 
\begin{align*}
    |\phi(x) - \phi^{\prime}(x) x| = \bigg | \frac{1+e^{-x} - x e^{-x}}{(1+e^{-x})^2} \bigg | \le \frac{1+e^{-x} +|x| e^{-x}}{(1+e^{-x})^2}.
\end{align*}
Note that $|x|e^{-x} \le e^{-2x} + 1$. We have 
\[|\phi(x) \le \frac{1+e^{-x} +e^{-2x} + 1}{(1+e^{-x})^2} \le 2.\]
\item \textbf{Tanh}. $\phi(x) = \frac{e^x- e^{-x}}{e^x + e^{-x}}\le 1$. Note that 
\begin{align*}
    \phi^{\prime}(x) = 1- \phi(x)^2\le 1 \\ 
    \phi^{\prime \prime}(x) = 2\phi(x)^3 - \phi(x) \le 2. 
\end{align*}
Besides, we know that 
\[
    |x \phi^{\prime}(x)| = \frac{4|x|}{(e^x + e^{-x})^2}\le 4. 
\]
Hence 
\[
|\phi(x) - \phi^{\prime}(x)| \le |\phi(x)| + |x \phi^{\prime}(x)| \le 5. 
\]
\item \textbf{SiLU}. Note that 
\begin{align*}
        \phi^{\prime}(x) = \frac{1+e^{-x} + xe^{-x}}{(1+e^{-x})^2},\\ 
    \phi^{\prime\prime}(x) = \frac{(2-x)e^{-x}}{(1+e^{-x})^2} + \frac{xe^{-2x}}{(1+e^{-x})^3}. 
\end{align*}
Because $|x|e^{-x} \le e^{-2x}+1$ and $|x|e^{-2x} \le e^{-3x} + 1$. We get $|\phi^{\prime}(x)| \le 2$ and $\phi^{\prime \prime}(x) | \le 4$. At last, 
\begin{align*}
    |\phi(x) - x \phi^{\prime}(x)| = \frac{|x|e^{-x}}{(1+e^{-x})^2}\le 1. 
\end{align*}
\item \textbf{Huberized ReLU}. It's obvious that $\phi^{\prime}(x)\le 1$ and $\beta = 1/h$. Note that $\phi$ is not second-order differentiable. At last, 
\[
|\phi(x) - x \phi^{\prime}(x) | = \begin{cases}
    0 & x<0,\\
    x^2/2h &0\le x\le h, \\
    h/2 &x >h. 
\end{cases}
\]
Hence, it's upper bounded by $h/2$. Now we have completed the proof of Example \ref{eg: homo act}.
\end{itemize}
\end{proof}

\subsection{Proof of Example \ref{eg: leak homo act}}
\label{sec:eg:leak-homo}
\begin{proof}[Proof of \Cref{eg: leak homo act}]
    Because for activation functions in \Cref{eg: homo act}, $\beta\le 4$ and $\rho \le 2$. Hence, for $\tilde \phi(x) = cx + (1-c)\phi(x)/4$, $\tilde \beta =1$ and $\rho=1$. Besides, since $0.5<c<1$, we must have $(\tilde \phi(x))^{\prime} \ge 0.25$. 
\end{proof}

\section{Additional Lemmas}

\begin{lemma}
\label{lem: bound for update lower} 
If $\frac{1}{2} \ge L_1 \ge \frac{1}{c}$ and  $L_2 \ge L_1 - L_1^2$,
we have 
\[
    L_2 \ge \frac{1}{c+2}.
\]
\end{lemma}
\begin{proof}[Proof of \Cref{lem: bound for update lower}]
For function $g(x) = x - x^2$, $g'(x) = 1-2x$. If $x \le \frac{1}{2}$, then $g(x)$ is increasing. Then 
\[
    g(L_1) \ge g( \frac{1}{c})  = \frac{c-1}{c^2} = \frac{c^2 + c-2}{c^2 (c+2)} \ge \frac{1}{c+2}. 
\]
Now we have completed the proof of \Cref{lem: bound for update lower}. 
\end{proof}

\begin{lemma}
\label{lem: near homo} 
Given a continuous function $f(x)$ s.t. $|f(x) - \langle \nabla f(x), x \rangle | \le \kappa$, then for a fixed constant $r>0$ there exists $C_{r,\kappa}$ and $C_r$ s.t. for any $\|x\|\ge r$, 
\[
    |f(x)| \le C_{r,\kappa} \|x\|, 
\] 
and for any $x$,
\[
    |f(x)| \le C_{r,\kappa}\|x\| + C_r.
\]
\end{lemma}
\begin{proof}[Proof of \Cref{lem: near homo}]
Since $f$ is continuous, let 
\[
    C_{r} = \max_{\|x\|=r}  |f(x)| /r. 
\]

Now for any $\|x\| >r$, let $y = \frac{rx}{\|x\|}$ and consider $g(s) = \frac{f(sy)}{s}$. Then we have 
\[
    g^\prime (s) = \frac{\langle \nabla f(sy), sy \rangle - f(sy) }{s^2}. 
\] 
Therefore, $ - \frac{\kappa}{s^2} \le g^\prime (s) \le \frac{\kappa}{s^2}$.  Let $s = \|x\|/r$, 
\begin{align*}
    \frac{f(x)r}{\|x\|} = g(s) &= g(1) + \int_{1}^s g^\prime (t) dt  \\ 
    &\le g(1) + \int_{1}^s  \frac{\kappa}{t^2} dt  \le g(1) + \kappa\\ 
    &\le rC_r + \kappa.  
\end{align*}
Therefore, $f(x) \le  (C_r + \frac{\kappa}{r} )\cdot \|x\|.$ Similarly, we can show that 
$-f(x) \le  (C_r + \frac{\kappa}{r} )\cdot \|x\|$. Therefore, for any $\|x\|\ge r$,
\[
    |f(x)| \le  (C_r + \frac{\kappa}{r} )\cdot \|x\|.
\]
Let $D=\max_{\|x\|\le r} |f(x)|$, we have for any $x$, 
\[
    |f(x)| \le (C_r + \frac{\kappa}{r})\cdot \|x\| + D. 
\]

We have completed the proof of \Cref{lem: near homo}. 
\end{proof}
\begin{lemma}
\label{lem: bound for x and 1/log(4nx)}
Fixing $c>1$, then for every  $0<x \le \frac{1}{c}$, we have 
\[
    x \le \frac{-1}{\log (cx)}. 
\]
\end{lemma}
\begin{proof}[Proof of \Cref{lem: bound for x and 1/log(4nx)}]
This is equivalent to show that 
\[
    x \log (cx) \ge -1. 
\]
Let $s(x) = x \log(cx)$, then $s^\prime (x) = 1 + \log(cx)$. Hence $s(x)$ is decreasing when $0<x<\frac{1}{ce}$ and is increasing when $x \ge \frac{1}{ce}$.  The minimum of $s(x)$ is achieved at $x = \frac{1}{ce}$, which is
\[
    s( 1/(ce))  =  -\frac{1}{ce} \ge -1.  
\]
This completes the proof of \Cref{lem: bound for x and 1/log(4nx)}. 
\end{proof}

\begin{lemma}
\label{lem: bound for the phi and psi}
Given $0<b\le\frac{1}{2}$ and $a>0$, we have 
\[
    \frac{1+a}{1-b} \le (1+2a+2b).
\]
\end{lemma}
\begin{proof}[Proof of \Cref{lem: bound for the phi and psi}]
This is equivalent to show that 
\[
    (1+a) \le (1+2a+2b)(1-b) = 1+ 2a +b - 2ab - 2b^2.
\]
This is equivalent to 
\[
    2b(a+b) \le (a+b). 
\]
Since $a+b>0$ and $ b\le \frac{1}{2}$, this is true.  Now we have completed the proof of \Cref{lem: bound for the phi and psi}.
\end{proof}

\begin{lemma} 
\label{lem:thresh bound}
    Given $c>e$, we have for any $x>2c\log c$,
    \[
        \frac{\log x}{x} \le \frac{1}{c}. 
    \]
\end{lemma}
\begin{proof}[Proof of \Cref{lem:thresh bound}]
    It's equivalent to show that $x - c\log x \ge 0$. Let $g(x) = x - c\log x$. $g'(x) = 1-c/x$. When $x>2c\log c >2c$, $g'(x)<0$. Hence, the minimal is $g(2c\log c)$. Note that 
    \[
    g(2c\log c) = 2c \log c - c\log c - c\log 2 - c\log \log c = c \log c - c\log 2 - c \log\log c = c\log  \frac{c}{2\log c}. 
    \]
    Now we want to show that $c>2\log c$. Let $h(y) = y - 2\log y$. $h'(y) = 1 - 2/y>0$ when $y>e$. $h(e) = e - 2>0$. Hence $h(c)>h(e)>0$ and $g(2c\log c) >0$. This leads to $g(x) >0.$
    Then, we complete the proof of \Cref{lem:thresh bound}. 
\end{proof}

\begin{lemma}
\label{lem:lowerbound l}
Given $\ell (x) = \log(1+e^{-x})$ and $c>0$, we have for any $x$,
\[
    \ell(x+c) \ge \exp(-c) \ell (x).  
\]
\end{lemma}
\begin{proof}[Proof of \Cref{lem:lowerbound l}]
Let $g(x) = \ell(x+c) - \exp (-c) \ell(x)$. Then, we have 
\[
    g^\prime (x) = \frac{-1}{1+\exp(x+c)} + \frac{1}{\exp(c) + \exp(x+c)} <0. 
\]
Therefore, $g(x)$ is monotonically decreasing. When $x\to \infty$, we have 
\[
    \lim_{x\to \infty} g(x) = \lim_{x\to \infty} [\ell(x+c) - \exp (-c) \ell(x)] = \exp(-x-c) -\exp(-c) \exp(-x) = 0. 
\]
Therefore, $g(x) \ge 0$ for any $x$. Now, we complete the proof of \Cref{lem:lowerbound l}.
\end{proof}

\begin{lemma}
\label{lem:class}
    Assume $\ell(x)= \log(1+e^{-x})$. If $\ell(x+c) \le 2 \ell(x)$, we have 
    \[
    c \le \ln 3 + |x|.
    \]
\end{lemma}
\begin{proof}[Proof of \Cref{lem:class}]
    Note that 
    \begin{align*}
        \ell(x+c) - 2 \ell(x) = \log \frac{1+e^{x+c}}{1+2e^x + e^{2x}} \le 0.
    \end{align*}
    Then, 
    \[
    \frac{1+e^{x+c}}{1+2e^x + e^{2x}}\le 1 \implies e^c \le 2 + e^x \le 2+ e^{|x|} \le 3 e^{|x|}. 
    \]
    Therefore, $c\le \ln3 + |x|$.  This completes the proof of \Cref{lem:class}.
\end{proof}

\end{document}